\documentclass[11pt,letterpaper]{article}
\usepackage[margin=1in]{geometry}
\usepackage{graphicx}
\usepackage{multirow}
\usepackage{amsmath,amssymb,amsfonts,mathtools}
\usepackage{amsthm}
\usepackage{bm}
\usepackage{array}
\usepackage{booktabs}
\usepackage{tabularx}
\usepackage{algorithm}
\usepackage{algorithmicx}
\usepackage{algpseudocode}
\usepackage{xcolor}
\usepackage{microtype}
\usepackage{tikz}
\usetikzlibrary{arrows.meta,positioning,fit,calc,patterns}
\usepackage{subcaption}
\usepackage{caption}
\usepackage{makecell}
\usepackage{hyperref}
\usepackage{cleveref}
\usepackage{adjustbox}
\usepackage{placeins}
\algrenewcommand\algorithmicrequire{\textbf{Input:}}
\algrenewcommand\algorithmicensure{\textbf{Output:}}
\algrenewcommand\algorithmiccomment[1]{\hfill$\triangleright$ #1}
\newif\ifdraftmode
\draftmodetrue

\newcommand{\V}{\bm{\mathcal{V}}}

\newcommand{\fW}{\mathfrak{W}}
\newcommand{\fB}{\mathfrak{B}}
\newcommand{\cP}{\mathcal{P}}
\newcommand{\cQ}{\bm{\mathcal{Q}}}
\newcommand{\cW}{\bm{\mathcal{W}}}
\newcommand{\cG}{\bm{\mathcal{G}}}
\newcommand{\cV}{\bm{\mathcal{V}}}
\newcommand{\cE}{\bm{\mathcal{E}}}
\newcommand{\cC}{{\bm{\mathcal{C}}}}
\newcommand{\cU}{{\bm{\mathcal{U}}}}
\newcommand{\cR}{{\bm{\mathcal{R}}}}
\newcommand{\cY}{\bm{\mathcal{Y}}}
\newcommand{\cS}{\bm{\mathcal{S}}}
\newcommand{\cX}{\bm{\mathcal{X}}}
\newcommand{\bS}{\bm{S}}
\newcommand{\bX}{\bm{X}}
\newcommand{\tprod}{\mathbin{\ast}}
\newcommand{\tdag}{\dagger}

\newcommand{\supp}{\operatorname{supp}}

\def\fro{\textnormal{F}}
\newcommand{\RITCURTC}{\rm{R-ItCUR}}
\newcommand{\ITCURTC}{\rm{ITCURTC}}
\newcommand{\RTIHT}{\rm{Robust-IHT}}
\newcommand{\RIHT}{\rm{Robust-IHT}}

\theoremstyle{plain}
\newtheorem{theorem}{Theorem}[section]

\newtheorem{lemma}[theorem]{Lemma}

\theoremstyle{definition}
\newtheorem{definition}[theorem]{Definition}
\theoremstyle{remark}
\newtheorem{remark}[theorem]{Remark}
\theoremstyle{definition}
\newtheorem{assumption}[theorem]{Assumption}

\title{Robust Low-Tubal-Rank Tensor Completion under Cross-Concentrated Sampling}
\author{HanQin Cai\thanks{School of Data, Mathematical, and Statistical Sciences and Department of Computer Science, University of Central Florida, Orlando, FL 32816, USA}, 
Longxiu Huang\thanks{Department of Computational Mathematics, Science, and Engineering and Department of Mathematics, Michigan State University, East Lansing, MI 48840, USA}, 
Jing Qin\thanks{Department of Mathematics, University of Kentucky, Lexington, KY 40506, USA}, 
and Chengyue Wu\thanks{Department of Imaging Physics, Department of Biostatistics, Department of Breast Imaging, and Institute for Data Science in Oncology, University of Texas MD Anderson Cancer Center, Houston, TX 77054, USA}
}
\date{}

\begin{document}
\maketitle

\begin{abstract}
Tensor cross-concentrated sampling (t-CCS) bridges entrywise sampling and t-CUR slice-wise sampling by observing entries only within selected horizontal and lateral slices. Existing t-CCS completion methods, however, assume that the observations are free of gross corruption. In this work, we study robust recovery of a third-order low-tubal-rank tensor from partial t-CCS observations contaminated by sparse, arbitrarily large outliers. We propose \textit{Robust Iterative t-CUR} (\RITCURTC), a tensor-native algorithm that partitions the sampled tensor cross into two exterior blocks and an intersection block, applies adaptive blockwise Welsch correction for outlier suppression, and updates the low-rank component through projected blockwise gradient descent. By operating directly on the sampled cross, \RITCURTC{} avoids reconstructing the full tensor throughout the iterations, resulting in substantial memory and computational savings. Experiments on synthetic tensors, cardiac MRI data, and three-dimensional seismic data demonstrate accurate recovery and strong robustness to sparse gross corruptions. The results further highlight the importance of explicitly exploiting the cross-concentrated sampling structure in robust tensor completion.
\end{abstract}



\section{Introduction}\label{sec:intro}
Tensors provide a natural representation for multi-way data arising in imaging, signal processing, and data science, including image processing \cite{liu2013tensor,kilmer2011factorization}, hyperspectral imagery~\cite{chang2017weighted,cai2021modewisedecomp}, seismic volumes~\cite{kreimer2012tensor}, and recommendation data~\cite{karatzoglou2010multiverse}. In many practical applications, however, such data are only partially observed because of acquisition constraints, transmission loss, or the high cost of dense measurements. Recovering the missing entries of a low-rank tensor from partial observations, known as the \textit{tensor completion} (TC) problem, has therefore become a fundamental task.

Since tensor rank is not uniquely defined, the choice of low-rank model plays a crucial role in both the formulation and performance of TC algorithms. Among the many notions of tensor rank, the \emph{tubal rank} induced by the tensor--tensor product (t-product) and the tensor singular value decomposition (t-SVD)~\cite{kilmer2011factorization} has proved particularly effective for third-order tensors with a distinguished mode, such as temporal or spectral dimensions. Moreover, the truncated t-SVD yields the optimal low-tubal-rank approximation in the least-squares sense~\cite{zhang2014novel,zhang2016exact,newman2025optimal}. Motivated by these advantages, we adopt the low-tubal-rank model throughout this paper.

\smallskip
\noindent\textbf{Cross-concentrated sampling.}
The recoverability of a low-rank tensor depends not only on the number of observed entries but also on \textit{how the observations are collected}. The dominant sampling model in the TC literature is entrywise Bernoulli or uniform sampling, under which exact recovery guarantees have been established for tensor nuclear norm minimization~\cite{zhang2014novel,zhang2016exact}. However, this model assumes that every entry is equally likely to be observed, an assumption that is often violated in practice. For example, in collaborative filtering, a small fraction of users may contribute most ratings~\cite{marlin2012collaborative}, while in medical imaging fully uniform acquisition may be physically infeasible or undesirable~\cite{kahn2014variable}. A complementary line of work considers \textit{structured} sampling schemes that observe entire fibers or slices, leading to fiber-CUR and t-CUR decompositions~\cite{cai2021modewisedecomp,chen2022tensor,ahmadi2024adaptive}. These methods reconstruct a tensor from a small subset of its own horizontal and lateral slices and inherit the interpretability and computational efficiency of classical matrix CUR decompositions~\cite{mahoney2009cur,drineas2008relative,hamm2020perspectives,hamm2020stability}. While slice-wise sampling is attractive when entire slices can be acquired efficiently, it may be prohibitively expensive or even infeasible when no slice is readily available in its entirety.

To bridge the gap between these two sampling models, \textit{Cross-Concentrated Sampling} (CCS) was recently introduced for matrices~\cite{cai2023ccs}. Its tensor counterpart, t-CCS, together with a nonconvex solver, was subsequently proposed for low-tubal-rank tensor completion~\cite{su2024tccs}. Under the t-CCS model, a subset of horizontal and lateral slices is first selected, after which entries are sampled only within those slices. Consequently, t-CCS interpolates between several existing sampling paradigms: it reduces to t-CUR sampling when the selected slices are densely observed, and to Bernoulli or uniform sampling when all slices are selected. An illustration of t-CCS is provided in \Cref{fig:1x4layout1}.

\begin{figure}[!th]
    \centering
    \begin{minipage}{0.22\linewidth}
        \includegraphics[width=\linewidth]{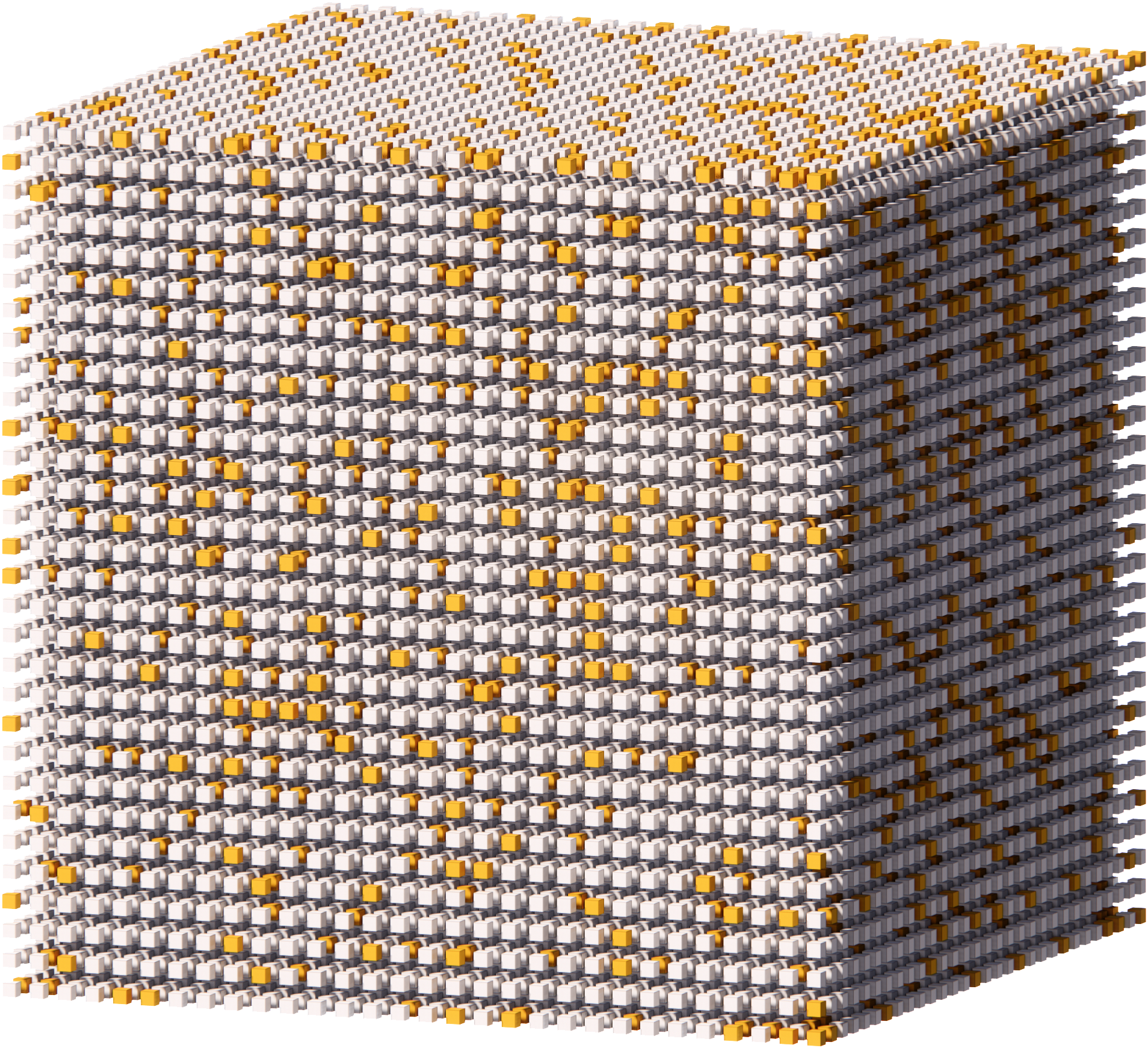}
        \subcaption*{\footnotesize{Bernoulli sampling}}
    \end{minipage}
    \hfill
    \begin{minipage}{0.22\linewidth}
        \includegraphics[width=\linewidth]{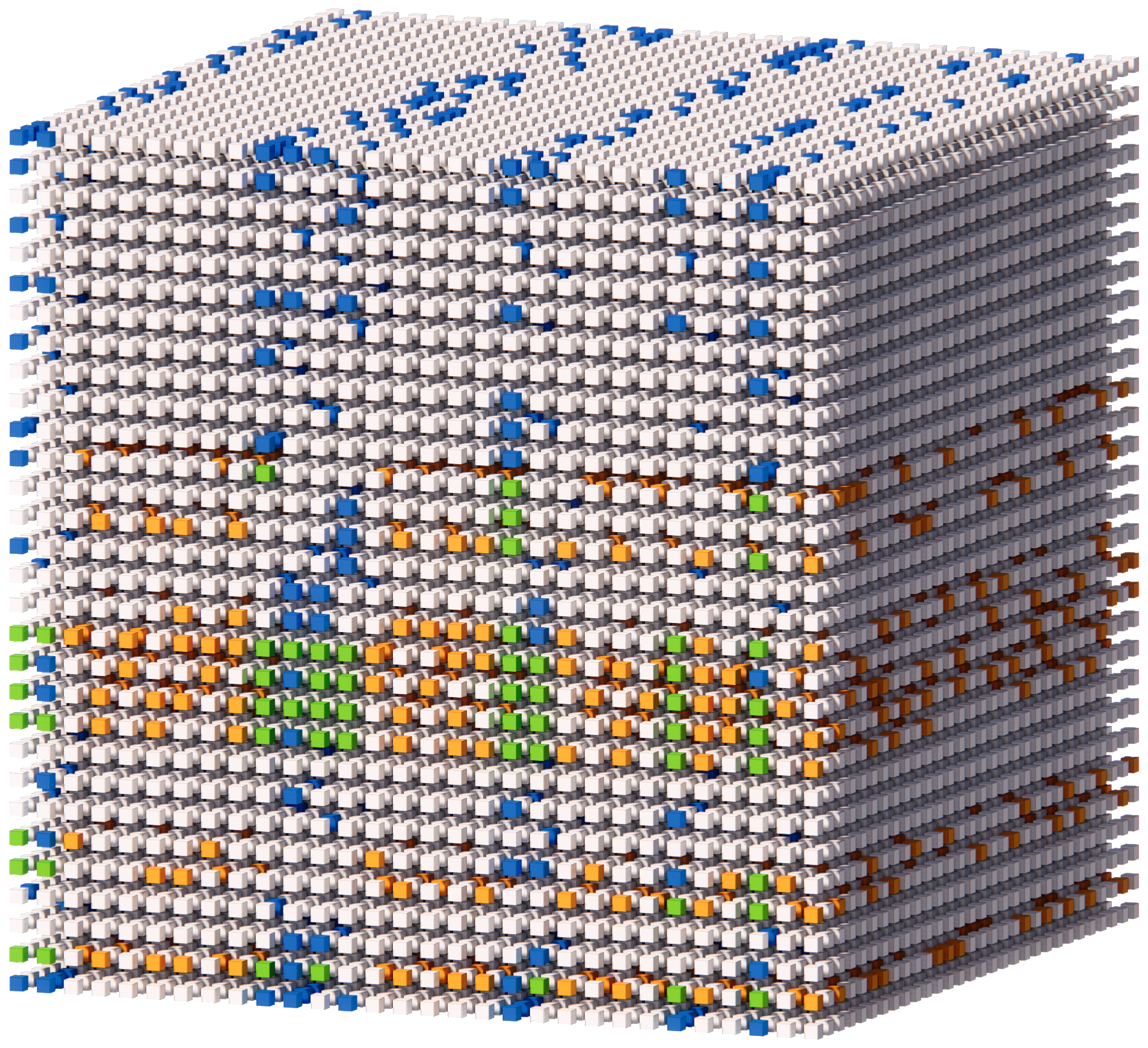}
        \subcaption*{\footnotesize{t-CCS less}}
    \end{minipage}
    \hfill
    \begin{minipage}{0.22\linewidth}
        \includegraphics[width=\linewidth]{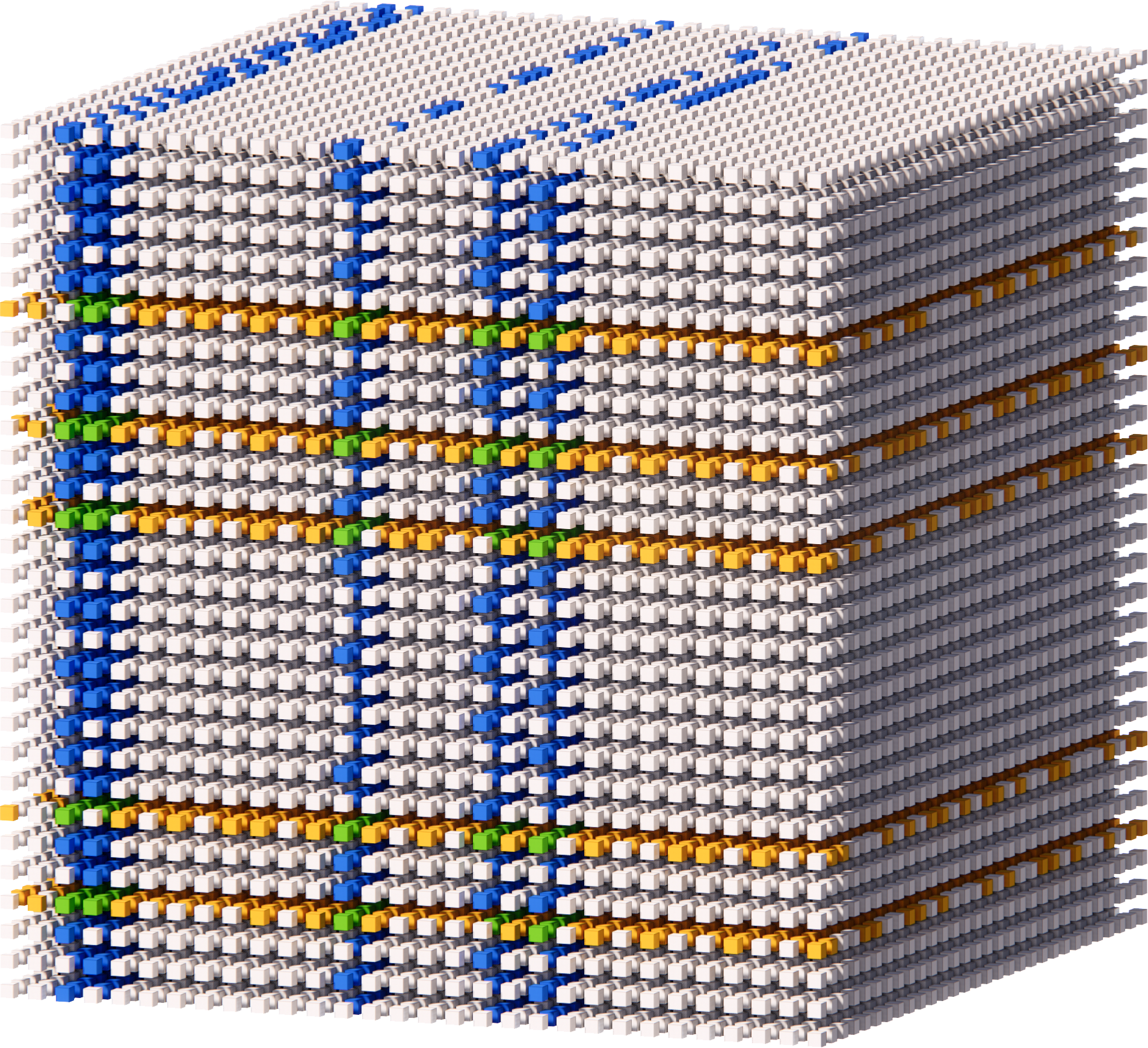}
        \subcaption*{\footnotesize{t-CCS more}}
    \end{minipage}
    \hfill
    \begin{minipage}{0.22\linewidth}
        \includegraphics[width=\linewidth]{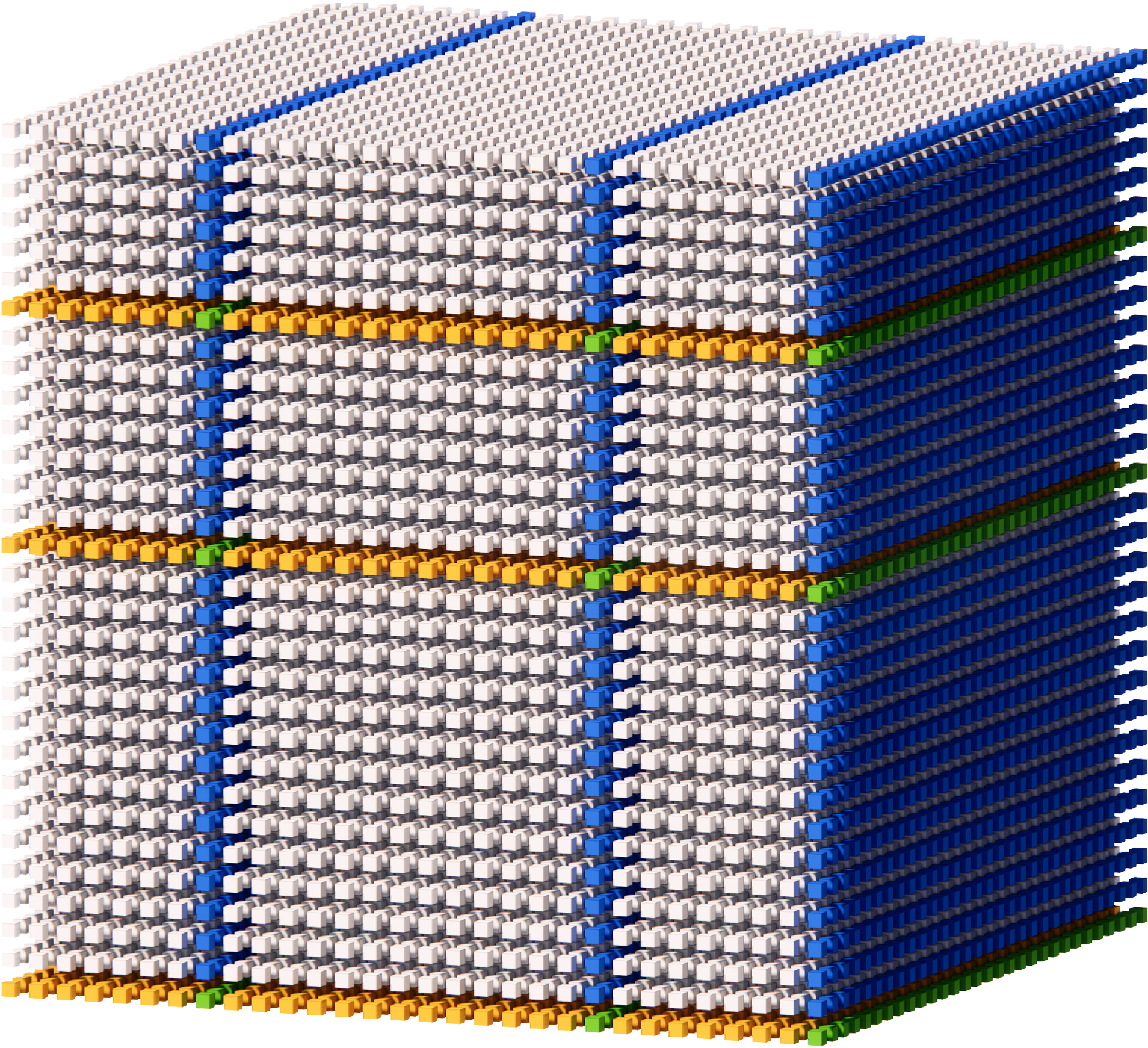}
        \subcaption*{\footnotesize{Tensor CUR sampling}}
    \end{minipage}
    \caption{\cite{cai2024rccs}. Illustration of four sampling strategies at the same total observation rate: Bernoulli, less concentrated t-CCS, more concentrated t-CCS, and tensor CUR. Blue, yellow and green grids mark the observed entries in the lateral, horizontal and intersected subtensors, respectively.}\label{fig:1x4layout1}
\end{figure}

\smallskip
\noindent\textbf{Sparse corruption.}
A second gap between existing theory and practice arises from data corruption. Real-world measurements are often corrupted not only by small dense noise but also by \textit{gross, sparse outliers}, such as dead pixels in images, saturated spectral measurements, corrupted seismic traces, and moving foreground objects in video sequences. For fully observed low-rank data, this challenge is addressed by robust principal component analysis (RPCA), which seeks to decompose the data into low-rank and sparse components. For matrices, RPCA has been studied extensively using both convex relaxations~\cite{candes2011robust} and scalable nonconvex algorithms~\cite{netrapalli2014nonconvex,yi2016fast,cai2019accaltproj,cai2021lrpca,tong2021accelerating}. Its tensor analogue under the t-product framework is known as tensor RPCA~\cite{lu2019tensor}.

The situation becomes substantially more challenging when the data are both incomplete and corrupted, particularly under structured sampling patterns. For matrices, the recent work~\cite{cai2024rccs} demonstrated that CCS remains robust to sparse outliers when coupled with a carefully designed solver. Whether a similar robustness property holds in the tensor setting, however, remains unknown. This naturally leads to the central question of this paper:
\begin{center}
\fbox{\begin{minipage}{0.8\linewidth}
\textit{Can low-tubal-rank tensors still be recovered under t-CCS sampling in the presence of sparse outliers?}
\end{minipage}}
\end{center}

\subsection{Related Work}
\label{subsec:related-work}

Robust data completion aims to recover a structured low-rank object from incomplete observations contaminated by sparse heavy-tailed errors. In the matrix setting, convex robust matrix completion methods combine nuclear-norm regularization for the low-rank component with an entrywise sparsity penalty for the corruptions~\cite{candes2011robust,chen2013low,klopp2017robust}. To improve scalability, nonconvex approaches instead alternate low-rank updates with thresholding or sparsification of the observed residual~\cite{yi2016fast,cherapanamjeri2017nearly,cai2026lrmc}. Robust-loss formulations offer another alternative by reducing the influence of large residuals through, for example, least absolute deviations and maximum correntropy \cite{li2020nonconvex,giampouras2025opsa,he2019robust}. Together, these methods establish the main algorithmic paradigms for jointly handling missing data and gross outliers. When the target possesses additional algebraic structure, however, explicitly incorporating that structure can lead to more specialized and efficient recovery methods.

Robust Hankel completion provides a representative example of structure-aware approaches. Robust-EMaC recovers the low-rank Hankel matrix from incomplete and corrupted samples using structured nuclear-norm minimization~\cite{chen2013spectral}. More recent nonconvex methods exploit the Hankel structure directly.
SAP uses structured alternating projections for fast low-rank Hankel recovery with guaranteed linear convergence~\cite{zhang2019correction}.
HSGD combines outlier estimation with factorized, structure-preserving low-rank updates~\cite{cai2023structured}, while HSNLD uses preconditioned updates to accelerate ill-conditioned robust Hankel recovery and achieves a linear convergence rate independent of the condition number~\cite{cai2025accelerating}. These developments illustrate the broader
benefit of adapting robust completion algorithms to the structure of the
target object. Hankel matrices provide matrix representations of
convolutional operators, while the t-product extends this convolutional
viewpoint to third-order tensors and underlies the tubal-rank model.

Within this broader setting, robust tensor completion has been studied under
several notions of tensor rank. Representative approaches use matrix unfoldings, CP factorization, tensor-ring representations, or nonconvex rank and sparsity surrogates~\cite{goldfarb2014robust,zhao2015bayesian,huang2020robust,zhao2022robust}. Among these models, the t-SVD framework is particularly relevant to the present work. For instance, \cite{jiang2019robust} proposed a convex formulation combining the tensor nuclear norm with a vectorized $\ell_1$ penalty and established exact-recovery guarantees. Nonconvex low-tubal-rank formulations were subsequently developed to reduce the bias and computational cost associated with convex surrogates \cite{wang2019robust}. Recent work~\cite{he2024correntropy} introduced a factorized model based on maximum correntropy and half-quadratic minimization. For a fixed kernel scale, maximizing Gaussian correntropy is equivalent, up to an additive constant and positive scaling, to minimizing a Welsch loss, making their formulation the closest robust-loss approach to the correction adopted in this paper. More recently, ScaledGD has provided condition-number-independent linear convergence guarantees for several low-tubal-rank estimation problems, including robust tensor completion \cite{wu2026guaranteed}. Despite their different optimization strategies, these methods are designed for observations sampled entrywise from the ambient tensor and do not exploit a tensor-cross representation.

CCS-based completion addresses this limitation from a complementary computational perspective. Matrix CCS interpolates between entrywise sampling and CUR row--column sampling \cite{cai2023ccs}, while Robust CUR Completion extends this observation model to matrix data contaminated by sparse outliers \cite{cai2024rccs}. In the tensor setting, t-CCS and \ITCURTC{} extend the same sampling flexibility to low-tubal-rank completion, but assume that the observations are free of gross corruptions \cite{su2024tccs}. A related but distinct method, Robust Tensor CUR, combines a tensor-CUR construction with sparse-corruption recovery \cite{cai2024robust}. However, it addresses robust decomposition of low-Tucker-rank tensors rather than low-tubal-rank completion from partial observations confined to a sampled tensor cross.

\subsection{Main Contributions}

The main contributions of this work are summarized as follows.

\begin{itemize}

\item We formulate a robust low-tubal-rank tensor completion problem under tensor cross-concentrated sampling (t-CCS), where observations are collected only within selected horizontal and lateral slices and may be contaminated by sparse gross outliers. The resulting sampled tensor cross naturally decomposes into three disjoint blocks, leading to a block-separable data-fidelity objective.

\item We propose \emph{Robust Iterative t-CUR} (\RITCURTC), a novel tensor-native algorithm for robust tensor completion under t-CCS. By maintaining an implicit t-CUR representation and updating only the sampled cross blocks, \RITCURTC{} avoids reconstructing the full tensor during the iterations and therefore achieves substantial memory and computational savings.

\item Extensive experiments on both synthetic and real-world datasets demonstrate the effectiveness and robustness of \RITCURTC. The results show that explicitly exploiting the cross-concentrated sampling structure significantly improves recovery performance in the presence of sparse gross outliers.

\end{itemize}

\section{Notation and Preliminaries}
\label{sec:prelim}

\noindent\textbf{Notation.}
We denote tensors by calligraphic capital letters, such as
$\cX$, matrices by uppercase letters, such as $\bX$,   vectors
by lowercase bold letters, such as $\bm{x}$, and scalars by regular letters, such as $x$.We use $\mathbf{0}$ to denote the zero vector or zero tensor, with the appropriate dimension inferred from the context, and $\mathbb{I}_n$ the $n\times n$ identity matrix. For a matrix $A$,  $A^*$ denotes its conjugate transpose, while $\cX^\top$ denotes the tensor conjugate transpose of $\cX$.
For a positive integer $n$, we denote
\(
    [n]:=\{1,\ldots,n\}.
\)
For a finite set $\bS$, its cardinality is denoted by $|\bS|$. If
$I\subseteq[n]$, then $I^{c}:=[n]\setminus I$ denotes its
complement.

For a third-order tensor
$\cX\in\mathbb R^{n_1\times n_2\times n_3}$,
$[\cX]_{i,j,k}$ denotes its $(i,j,k)$-th entry,
$[\cX]_{i,j,:}$ its $(i,j)$-th tube, and
$[\cX]_{:,:,k}$ its $k$-th frontal slice. The horizontal and
lateral slices of $\cX$ are denoted by
$[\cX]_{i,:,:}$ and $[\cX]_{:,j,:}$, respectively.
For index sets $I\subseteq[n_1]$ and $J\subseteq[n_2]$, we write
\(
    [\cX]_{I,:,:},
    [\cX]_{:,J,:},
    [\cX]_{I,J,:}
\)
for the corresponding subtensors.
The discrete Fourier transform along the third mode is denoted by
\[
    \widehat{\cX}
    :=
    \operatorname{fft}(\cX,[],3),
\]
and
\(
    \widehat{\bX}^{(k)}
    :=
    [\widehat{\cX}]_{:,:,k}
\)
denotes its $k$-th frontal slice. The Frobenius norm and the
infinity norm of $\cX$ are defined by
\[
    \|\cX\|_{\fro}
    :=
    \bigg(
        \sum_{i,j,k}
        |[\cX]_{i,j,k}|^2
    \bigg)^{1/2},
    \qquad
    \|\cX\|_{\infty}
    :=
    \max_{i,j,k}|[\cX]_{i,j,k}|.
\]
For a matrix $A$, we use $A^*$ for its conjugate transpose and
$A^\dagger$ for its Moore--Penrose inverse. For a tensor,
$\cX^\top$ denotes the tensor conjugate transpose defined below.
The Kronecker product is denoted by $\otimes$, and $\bm{e}_i$
denotes the $i$-th standard basis vector of the appropriate dimension.
For an index set
\(
    \Omega\subseteq[n_1]\times[n_2]\times[n_3],
\)
the sampling operator $\mathcal P_\Omega$ is defined entrywise by
\begin{equation}
    [\mathcal P_\Omega(\cX)]_{i,j,k}
    =
    \begin{cases}
        [\cX]_{i,j,k},
        &(i,j,k)\in\Omega;\\
        0,
        &(i,j,k)\notin\Omega.
    \end{cases}
    \label{eq:sampling-operator}
\end{equation}
The support of a tensor $\cS$ is denoted by
\[
    \supp(\cS)
    :=
    \{(i,j,k):[\cS]_{i,j,k}\neq0\}.
\]

\smallskip
\begin{definition}[t-product~\cite{kilmer2011factorization}]
For
$\cX\in\mathbb R^{n_1\times n_2\times n_3}$, let
\(
    \operatorname{bcirc}(\cX)
    \in
    \mathbb R^{n_1n_3\times n_2n_3}
\)
denote the block-circulant matrix formed from its frontal slices, and
let $\operatorname{unfold}(\cdot)$ and $\operatorname{fold}(\cdot)$ denote the
standard stacking and inverse-stacking operators. The
\emph{t-product} of
\(
    \cX\in\mathbb R^{n_1\times n_2\times n_3}
    \text{~and~}
    \cG\in\mathbb R^{n_2\times n_4\times n_3}
\)
is defined by
\begin{equation*}
    \cX*\cG
    :=
    \operatorname{fold}
    \left(
        \operatorname{bcirc}(\cX)
        \operatorname{unfold}(\cG)
    \right).
\end{equation*}
\end{definition}
A block-circulant matrix can be block-diagonalized by the DFT. More
precisely, if $F_{n_3}$ denotes the DFT matrix consistent with the
Fourier-transform convention above, then
\begin{equation}
\begin{aligned}
    \overline{\cX}
    &:=
    (\bm{F}_{n_3}\otimes \mathbb{I}_{n_1})
    \operatorname{bcirc}(\cX)
    (\bm{F}_{n_3}^{-1}\otimes \mathbb{I}_{n_2})=
    \operatorname{blkdiag}
    \bigl(
        \widehat \bX^{(1)},
        \ldots,
        \widehat \bX^{(n_3)}
    \bigr).
\end{aligned}
\label{eq:blockdiag}
\end{equation}
The right-hand side of \eqref{eq:blockdiag} is the block-diagonal matrix with diagonal blocks $\widehat{\bX}^{(1)}$,..., $\widehat{\bX}^{(n_3)}$. Consequently, the t-product reduces to ordinary matrix multiplication
on the Fourier slices:
\(
    \widehat{(\cX*\cG)}^{(k)}
    =
    \widehat \bX^{(k)}
    \widehat{\bm{G}}^{(k)},
    \forall\, k\in[n_3].
\)

The identity tensor $\mathcal I\in\mathbb R^{n\times n\times n_3}$ has $\mathbb{I}_n$ as its first frontal slice and zero matrices as all remaining frontal slices. Equivalently, $\widehat I^{(k)}=\mathbb{I}_n$ for every $k\in[n_3]$, and therefore
\(
    \mathcal I*\cX =  \cX*\mathcal I = \cX.
\)
The tensor conjugate transpose
$\cX^\top\in
\mathbb R^{n_2\times n_1\times n_3}$
is obtained by conjugate-transposing each frontal slice and reversing
the order of slices $2,\ldots,n_3$. Equivalently,
\(
    \overline{\cX^\top}
    =
    \overline{\cX}^{\,*},
\)
where $\overline{\cX}^{\,*}$ is the complex conjugate of the block-diagonal matrix $\overline{\cX}$ defined in \eqref{eq:blockdiag}.
A tensor is called \emph{$f$-diagonal} if every frontal slice is a
diagonal matrix. A square tensor
$\cQ\in\mathbb R^{n\times n\times n_3}$
is called \emph{orthogonal} if
\(
    \cQ^\top*\cQ
    =
    \cQ*\cQ^\top
    =
    \mathcal I.
\)
Equivalently, each Fourier slice $\widehat{\bm{Q}}^{(k)}$ is unitary.

\smallskip
\begin{definition}[Tubal rank \cite{kilmer2011factorization}]
For $\cX\in\mathbb R^{n_1\times n_2\times n_3}$, define
\(
    r_k
    :=
    \operatorname{rank}(\widehat X^{(k)}),
    k\in[n_3].
\)
The vector
\begin{equation}
    \operatorname{rank}_m(\cX)
    :=
    \bm{r}
    =
    (r_1,\ldots,r_{n_3})
    \label{eq:multirank}
\end{equation}
is called the \emph{multi-rank} of $\cX$, and
\[
    \operatorname{rank}_t(\cX)
    :=
    \|\bm{r}\|_\infty
    =
    \max_{k\in[n_3]}r_k
\]
is called its \emph{tubal rank}. Moreover, \(
    \|\bm{r}\|_1 = \sum_{k=1}^{n_3}r_k =\operatorname{rank}(\overline{\cX})
\)
is the aggregate Fourier-domain rank.
\end{definition}

\smallskip
\begin{definition}[t-SVD~\cite{kilmer2011factorization}]
    The \emph{t-SVD}  of
$\cX$ is a factorization
\begin{equation}
    \cX
    =
    \cW*
    \boldsymbol{\Theta}*
    \cV^\top,
    \label{eq:t-svd}
\end{equation}
where
\(
    \cW
    \in\mathbb R^{n_1\times n_1\times n_3},
    \cV
    \in\mathbb R^{n_2\times n_2\times n_3}
\)
are orthogonal tensors and
$\boldsymbol{\Theta}\in
\mathbb R^{n_1\times n_2\times n_3}$
is $f$-diagonal. Equivalently, for every $k\in[n_3]$,
\(
    \widehat \bX^{(k)}
    =
    \widehat{\bm{W}}^{(k)}
    \widehat\Theta^{(k)}
    \widehat{\bm{V}}^{(k)*}
\)
is an ordinary matrix SVD.

If $\operatorname{rank}_t(\cX)=r$, its compact t-SVD is written
as
\(
    \cX
    =
    \cW*
    \boldsymbol{\Theta}*
    \cV^\top,
    \)
where
\(
    \cW\in\mathbb R^{n_1\times r\times n_3},
    \boldsymbol{\Theta}
        \in\mathbb R^{r\times r\times n_3},
    \cV\in\mathbb R^{n_2\times r\times n_3}.
\)
\end{definition}
The Moore--Penrose t-inverse $\cX^\dagger$ is defined by
\(
    \widehat{\bX^\dagger}^{(k)}
    =
    \bigl(\widehat{\bX}^{(k)}\bigr)^\dagger,
    k\in[n_3].
\)
The tensor spectral norm and condition number are defined as
\[
    \|\cX\|
    :=
    \|\operatorname{bcirc}(\cX)\|_2
    =
    \max_{k\in[n_3]}
    \|\widehat{\bX}^{(k)}\|_2,
    \qquad
    \kappa(\cX)
    :=
    \|\cX\|\,
    \|\cX^\dagger\|.
\]

\smallskip
\noindent\textbf{The t-CUR decomposition.}
The following result provides the algebraic basis for the proposed
cross-based reconstruction method.

\smallskip
\begin{theorem}[Exact t-CUR decomposition
{\cite{su2024tccs,hamm2020perspectives}}]
\label{thm:tcur}
Let
$\cX\in\mathbb R^{n_1\times n_2\times n_3}$
have multi-rank $\bm{r}$. For
$I\subseteq[n_1]$ and $J\subseteq[n_2]$, define
\[
    \cR
    :=
    [\cX]_{I,:,:},
    \qquad
    \cC
    :=
    [\cX]_{:,J,:},
    \qquad
    \cU
    :=
    [\cX]_{I,J,:}.
\]
Then \( \cX
    =
    \cC*
    \cU^\dagger*
    \cR\)
if and only if
\(
    \operatorname{rank}_m(\cR) = \operatorname{rank}_m(\cC)
    = \bm{r}.
\)
Equivalently, the intersection tensor satisfies
$\operatorname{rank}_m(\cU)=\bm{r}$.
\end{theorem}

\section{Problem Formulation}\label{sec:model}

Consider the corrupted data tensor:
\begin{equation}\label{eq:model}
\cY=\cX^\star+\cS^\star,
\end{equation}
where $\cX^\star\in\mathbb{R}^{n_1\times n_2\times n_3}$ is a low-tubal-rank tensor with $\operatorname{rank}_t(\cX^\star)=r$ and $\cS^\star$ is a sparse tensor representing gross outliers. Rather than observing $\cY$ in its entirety, only a subset of its entries is available according to the t-CCS sampling pattern over the index set $\Omega_{\cR}\cup\Omega_{\cC}$, which is defined in \Cref{sec:sampling}. The \emph{robust t-CCS completion} problem aims to
recover $\cX^\star$ from these incomplete and corrupted observations by solving:
\begin{equation}\label{eq:robust t-ccs}
\min_{\cX,\,\cS}\
\big\|\mathcal{P}_{\Omega_{\cR}\cup\Omega_{\cC}}
(\cX+\cS-\cY)\big\|_\fro^2
\ \ \text{s.t.}\ \
\mathrm{rank}_t(\cX)=r \textnormal{ and }
\cS\ \text{is $\alpha$-sparse},
\end{equation}
where $\mathcal{P}_\Omega$ denotes the projection operator onto the sampling set $\Omega$, defined in \eqref{eq:sampling-operator}, and the precise definition of $\alpha$-sparsity is given in \Cref{sec:sparsity}. Note that when
$\cS=\bm{0}$, i.e., when the sparse component is absent,  \eqref{eq:robust t-ccs} reduces to the standard t-CCS completion problem studied in \cite{su2024tccs}.

\subsection{General t-CCS Procedure}\label{sec:sampling}

The t-CCS model proceeds in two stages. First, a subset of horizontal and lateral slices is selected, indexed by $I$ and $J$ respectively. Second, entries are sampled within the selected slices. In both stages, one may employ either uniform random sampling, which is suitable for incoherent and well-conditioned problems, or biased sampling strategies, such as leverage score methods, when prior structural information is available. For notational convenience, let $\pi_I$ and $\pi_J$ denote the inclusion probabilities used for slice selection along the horizontal and lateral modes, respectively, and let $\pi_\cR$ and $\pi_\cC$ denote the inclusion probabilities governing entrywise sampling within the selected horizontal and lateral slices.
The resulting t-CCS sampling pattern consists of two index sets, $\Omega_{\cR}$ and $\Omega_{\cC}$, corresponding to samples collected from the selected horizontal and lateral subtensors, respectively. The combined observation is denoted by $\mathcal{P}_{\Omega_{\cR}\cup\Omega_{\cC}}
(\cY)$, or the equivalent shorthand $\cY_{\Omega_{\cR}\cup\Omega_{\cC}}$. This procedure is summarized in \Cref{proc:sampling}.

\begin{algorithm}
\caption{General t-CCS Sampling}
\label{proc:sampling}
\begin{algorithmic}[1]
\State \textbf{Input:} $\cY$: corrupted data tensor; $\pi_I,\pi_J$: inclusion probabilities for slice selection; $\pi_\cR, \pi_\cC$: inclusion probabilities for entrywise sampling within the selected slices.
\State Select slice index sets $I:= \{\, i \in[n_1]: \mathrm{Bernoulli}(\pi_I(i)) = 1 \,\}$ and $J:= \{\, j \in[n_2]: \mathrm{Bernoulli}(\pi_J(j)) = 1 \,\}$.
\State Set $\cR=[\cY]_{I,:,:}$ and
$\cC=[\cY]_{:,J,:}$.
\State Sample entry sets $\Omega_{\cR}:=\{\, (i,j,k) \in I\times[n_2]\times[n_3]: \mathrm{Bernoulli}(\pi_\cR(i,j,k)) = 1 \,\}$ and $\Omega_{\cC}:=\{\, (i,j,k) \in[n_1]\times J\times[n_3]: \mathrm{Bernoulli}(\pi_\cC(i,j,k)) = 1 \,\}$.
\State \textbf{Output:}
$[\cY]_{\Omega_{\cR}\cup\Omega_{\cC}}$,
$\Omega_{\cR},\Omega_{\cC}$, $I$, $J$: cross-concentrated observations and their associated index sets.
\end{algorithmic}
\end{algorithm}

\subsection{Tube-wise Sparsity and Fourier Support Preservation}
\label{sec:sparsity}

The outlier model is central to this work. Since the low-rank component is modeled through the t-SVD framework, a suitable sparsity assumption must be preserved under the Fourier transform along the third mode. In contrast, the conventional entrywise sparsity model does not enjoy this property: a single corrupted entry in the spatial domain generally spreads across all Fourier slices after transformation. To overcome this difficulty, we adopt a \emph{tube-wise sparsity} model, in which an entire tube is either clean or corrupted, while only a small fraction of tubes may be corrupted. More formally:

\smallskip
\begin{assumption}[Tube-wise $\alpha$-sparsity]\label{ass:sparse}
Let the tube support of $\cS$ be
$\mathrm{supp}_t(\cS):=\{(i,j): [\cS]_{i,j,:}\neq \mathbf{0}\}
\subseteq[n_1]\times[n_2]$. Then $\cS$ is tube-wise
$\alpha$-sparse if at most an $\alpha$ fraction of tubes are corrupted in
every horizontal and lateral direction, i.e.\ for all $i\in[n_1]$,
$j\in[n_2]$,
\[
\big|\{j':(i,j')\in\mathrm{supp}_t(\cS)\}\big|\le\alpha n_2,
\qquad
\big|\{i':(i',j)\in\mathrm{supp}_t(\cS)\}\big|\le\alpha n_1 .
\]
\end{assumption}

Note that no randomness is assumed on the tube support pattern of $\cS$, and the outlier magnitudes may be arbitrarily large. Thus, the model permits adversarial sparse corruptions.
The key advantage of tube-wise sparsity is that it is preserved under the Fourier transform, which underlies the t-SVD framework.

\smallskip
\begin{lemma}[Fourier support preservation]\label{lem:fourier-support}
If $\cS$ is tube-wise $\alpha$-sparse, then for every $k\in[n_3]$, the Fourier slice $\widehat{\bS}^{(k)}$ always has support contained in $\mathrm{supp}_t(\cS)$. Therefore, each $\widehat{\bS}^{(k)}$ is
$\alpha$-sparse per row and per column, with the \emph{same} support set across all $k$.
\end{lemma}

\begin{proof}
For each $(i,j)$, the Fourier tube $[\widehat{\cS}]_{i,j,:}$ is the DFT of the spatial tube $[\cS]_{i,j,:}$. If $(i,j)\notin
\mathrm{supp}_t(\cS)$, then $[\cS]_{i,j,:}=\textbf{0}$, hence
$[\widehat{\cS}]_{i,j,:}=\textbf{0}$ and in particular
$\widehat{\bS}^{(k)}_{i,j}=0$ for all $k$. Thus
$\mathrm{supp}(\widehat{\bS}^{(k)})\subseteq\mathrm{supp}_t(\cS)$ for
each $k$. Since the support of each Fourier slice is contained in the common
tube support, the per-row and per-column sparsity bounds in
Assumption~\ref{ass:sparse} hold uniformly for each Fourier slice. This finishes the proof.
\end{proof}

Lemma~\ref{lem:fourier-support} shows that tube-wise sparse corruption induces a common-support sparse perturbation across all Fourier slices. Consequently, under the block-diagonal representation \eqref{eq:blockdiag}, the robust tensor completion problem can be viewed as a collection of matrix recovery problems, one for each frequency. This observation suggests that robust CUR techniques may be applied frequency-by-frequency. However, treating each Fourier slice independently fails to fully exploit the coupling structure imposed by the t-SVD framework and may lead to information loss across frequencies. Therefore, rather than solving a sequence of independent matrix problems, we seek a tensor-native approach that preserves the tensor structure throughout the recovery process. This motivates the development of the proposed algorithm, presented in the next section.

\section{Proposed Method}\label{sec:theory}

In this section, we present the proposed Robust Iterative t-CUR ({\RITCURTC}) method for efficiently solving the robust tensor completion problem \eqref{eq:robust t-ccs}.
Recall that the subtensor $\cU$ corresponds to the overlap between the sampled horizontal and lateral subtensors $\cR$ and $\cC$. To improve both computational and memory efficiency, R-ItCUR stores only the three disjoint blocks of the sampled cross and never explicitly reconstructs the full tensor during the iterations.

We begin by establishing the notation used for the algorithm development.
Given the index sets $\Omega_\cR$, $\Omega_\cC$, $I$, and $J$ returned from \Cref{proc:sampling}, define the complementary index sets
$I^{c}=[n_{1}]\setminus I$ and $J^{c}=[n_{2}]\setminus J$. Since the lateral and horizontal subtensors share the block indexed by $I\times J\times [n_3]$, we assign these overlapping observations to the intersection and partition the observation set
$\Omega=\Omega_{\cR}\cup\Omega_{\cC}$ into three disjoint index sets:
\begin{align*}
\Omega_{\cR'} &= \Omega_\cR\cap(I\times J^c\times[n_3]),\\
\Omega_{\cC'} &= \Omega_\cC\cap(I^c\times J\times[n_3]),\\
\Omega_{\cU} &= (\Omega_\cR\cup\Omega_\cU)\cap(I\times J\times[n_3]).
\end{align*}
Correspondingly, for any tensor $\cX$, we define the associated subtensors:
\begin{align*}
\cX_{\cR'} &= [\cX]_{I, J^c,:},&
\cX_{\cC'} &= [\cX]_{I^c, J,:},&
\cX_{\cU} &= [\cX]_{I, J,:}.
\end{align*}
Since the three blocks are mutually disjoint, the objective function in \eqref{eq:robust t-ccs} decomposes as:
\begin{align} \label{eq:blockloss}
\|\cP_{\Omega_{\cR}\cup\Omega_{\cC}}(\cX+\cS-\cY)\|_{\fro}^2
&=\|\cP_{\Omega_{\cR'}}(\cX_{\cR'}+\cS_{\cR'}-\cY_{\cR'})\|_{\fro}^2 \cr
&~~+\|\cP_{\Omega_{\cC'}}(\cX_{\cC'}+\cS_{\cC'}-\cY_{\cC'})\|_{\fro}^2
+\|\cP_{\Omega_{\cU}}(\cX_{\cU}+\cS_{\cU}-\cY_{\cU})\|_{\fro}^2\cr
&=\sum_{b\in\fB} \left\|\cP_{\Omega_b}(\cX_b+\cS_b-\cY_b)\right\|_{\fro}^2,
\end{align}
where we write
\[
    \fB=\{\cR',\cC',\cU\}
\]
for ease of notation.
Hence, certain optimization operations, such as thresholding and gradient descent steps, can be performed independently on each block.
Next, we will describe the resulting optimization procedure, which is summarized in \Cref{alg:ritcurtc}.

\begin{algorithm}[ht]
\caption{Robust Iterative t-CUR (\RITCURTC) for Robust t-CCS Tensor Completion}
\label{alg:ritcurtc}
\small
\begin{algorithmic}[1]
\State \textbf{Input:} $[\cY]_{\Omega_{\cR}\cup\Omega_{\cC}}$,
$\Omega_{\cR},\Omega_{\cC}$, $I$, $J$: cross-concentrated observations and their associated index sets; $r$: target rank; $\alpha$: contamination fraction; $\{\eta_b\}$: step sizes.
\ForAll{$b\in\mathfrak B$}
  \State Initialize $\cX_b^0$ by
  \eqref{eq:robust-initialization}--\eqref{eq:init X_0}
\EndFor

\For{$\ell=1,2,\ldots,L$}
  \ForAll{$b\in\mathfrak B$}
    \State Compute the robust parameter $\sigma_{\alpha,b}^\ell$ by \eqref{eq:welsch-scale}
    \State $\cS_b^\ell\gets
    \fW_{\sigma_{\alpha,b}^\ell}(\cP_{\Omega_b} (\cY_b-\cX_b^{\ell-1}))$
    \State $\widetilde{\cX}_b^\ell
    \gets \cX_b^{\ell-1} - \eta_b\, \cP_{\Omega_b} (\cX_b^{\ell-1}+\cS_b^\ell -\cY_b)$
  \EndFor
  \State $\cX_{\cU}^\ell =: \cW^\ell\tprod\boldsymbol{\Theta}^\ell\tprod(\cV^\ell)^\top
    \gets \textrm{t-SVD}_r(\widetilde{\cX}_{\cU}^\ell)$
  \State $\cX_{\cR'}^\ell \gets \cW^\ell\tprod(\cW^\ell)^\top\tprod \widetilde{\cX}_{\cR'}^\ell$
  \State $\cX_{\cC'}^\ell \gets
    \widetilde{\cX}_{\cC'}^\ell\tprod\cV^\ell\tprod(\cV^\ell)^\top$
\EndFor
\State $\cX_{\cR}^L\gets$ Concatenation of $\cX_{\cU}^L$  and $\cX_{\cR'}^L$
\State $\cX_{\cC}^L\gets$ Concatenation of $\cX_{\cU}^L$ and  $\cX_{\cC'}^L$
\State \textbf{Output:} $\cX_L=\cX_{\cC}^L\tprod(\cX_{\cU}^L)^{\tdag}\tprod\cX_{\cR}^L$: cross blocks of low-tubal-rank estimate.
\end{algorithmic}
\end{algorithm}

\subsection{Outlier Detection}\label{subsection:outlierDet}
Our optimization framework is compatible with many classic outlier detection modules developed for RPCA, such as hard-thresholding \cite{netrapalli2014nonconvex,cai2019accaltproj}, soft-thresholding \cite{cai2021lrpca,cai2026lrmc}, and sparsification \cite{yi2016fast,tong2021accelerating}. In this work, however, we employ the smooth Welsch correction function for outlier detection~\cite{holland1977robust,black1996unification}.
The primary motivation is to reduce sensitivity to parameter selection and avoid the unstable support estimates that may arise in classic methods when the residuals exhibit ambiguous or clustered patterns. The Welsch correction is a smooth, bounded, and nonconvex robust loss that gradually suppresses the influence of large residuals while retaining sensitivity to small residuals. Consequently, it provides a stable mechanism for separating sparse outliers from the low-rank component and is generally less sensitive to tuning parameters. Moreover, its smooth weighting scheme mitigates repeated support switching during the iterations, which is particularly beneficial under t-CCS sampling, where large residuals may arise from both sparse corruption and incomplete observations in the early stages of optimization.

For each $b\in\fB$, let $\cX_b^{\ell-1}$ be the estimate of the corresponding clean cross block at the previous iteration and denote the observed residual
\begin{equation*}\cE_b^\ell = \cP_{\Omega_b} \left(\cY_b-\cX_b^{\ell-1}\right).
\end{equation*}
The Welsch correction is defined as:
\begin{equation*}\fW_{\sigma}(\cE)=
    \left[
       \mathbf 1-\exp\left(-\frac{\cE\odot\cE}{2\sigma^2}\right)
    \right]\odot\cE,
\end{equation*}
where $\sigma>0$ is a robustness parameter, and the exponential $\exp(\cdot)$ and products $\odot$ are entrywise.
The blockwise robust
correction of \RITCURTC{} at $\ell$-th iteration is then given as:
\begin{equation} \label{eq:block-welsch-correction}
    \cS_b^\ell = \fW_{\sigma_{\alpha,b}^\ell}(\cE_b^\ell), \qquad \forall\, b\in\mathfrak B.
\end{equation}
Specifically, the robustness parameters are chosen adaptively as:
\begin{equation} \label{eq:welsch-scale}
    \sigma_{\alpha,b}^\ell =\max\left\{
        Q_{1-\alpha}
        \left(
          \left\{
          |[\cE_b^\ell]_{i,j,k}|:
          (i,j,k)\in\Omega_b
          \right\}
        \right),
        \sigma_{\min}
    \right\},
\end{equation}
where $Q_q$ denotes the empirical $q$-quantile and $\sigma_{\min}>0$ is a small constant used to prevent division by zero.
When the contamination fraction $\alpha=0$, \RITCURTC{} sets $\cS_b^\ell=\bm{0}$ directly for all $b\in\fB$, reducing the method to the non-robust t-CUR completion algorithm proposed in \cite{su2024tccs}.

\smallskip
\begin{remark}
The Welsch correction is generally not exactly sparse. Consequently, $\cS_b^\ell$ should be interpreted as an outlier-robust correction term rather than a hard estimate of the outlier support.
This design avoids unstable support switching during the iterations and reduces sensitivity to parameter tuning.
If an explicit support estimate is desired, it may be obtained through a separate thresholding or post-processing step after convergence.
\end{remark}

\subsection{Projected Blockwise Descent} \label{sec:projected GD}

We next update the low-tubal-rank component via projected blockwise gradient descent. Recall from \eqref{eq:blockloss} that the objective function decomposes into three disjoint blocks. Consequently, the gradient update for $\cX$ can be performed independently on each block, allowing for parallel implementation:
\begin{equation} \label{eq:observed-block-update}
    \widetilde{\cX}_b^\ell
    = \cX_b^{\ell-1} - \eta_b\, \cP_{\Omega_b} \left(\cX_b^{\ell-1}+\cS_b^\ell -\cY_b\right),
    \qquad \forall b\in\fB,
\end{equation}
where $\eta_b$ denotes the step size associated with block $b$.
The updated blocks generally do not satisfy the desired low-tubal-rank structure. A projection step is therefore required. However, forming the full tensor and computing a global t-SVD would negate the computational and memory advantages of the t-CCS framework. Instead, following the philosophy of t-CUR methods, all computations are performed directly on the sampled cross.

We first compute a rank-$r$ truncated t-SVD of the intersection block:
\begin{equation} \label{eq:current-core-tsvd}
    \cX_{\cU}^\ell
    =\textrm{t-SVD}_r(\widetilde{\cX}_{\cU}^\ell)
    = \cW^\ell\tprod\boldsymbol{\Theta}^\ell\tprod(\cV^\ell)^\top,
\end{equation}
where $\cW^\ell\in\mathbb R^{|I|\times r\times n_3}$ and $\V^\ell\in\mathbb R^{|J|\times r\times n_3}$ have orthonormal tubal columns.
The resulting factors define the dominant tubal column and row subspaces of the overlap block. We then project the exterior blocks onto these subspaces:
\begin{align} \label{eq:project blocks}
    \cX_{\cR'}^\ell = \cW^\ell\tprod(\cW^\ell)^\top\tprod \widetilde{\cX}_{\cR'}^\ell \quad\textnormal{ and }\quad
    \cX_{\cC'}^\ell =
    \widetilde{\cX}_{\cC'}^\ell\tprod\cV^\ell\tprod(\cV^\ell)^\top.
\end{align}
Conceptually, the updated low-rank tensor can be expressed through the t-CUR reconstruction:
\begin{align} \label{eq:low rank update}
\cX^\ell=\cX_{\cC}^\ell*(\cX_{\cU}^\ell)^\dagger*\cX_{\cR}^\ell,
\end{align}
where $\cX_{\cC}^\ell$ and $\cX_{\cR}^\ell$ denote the lateral and horizontal sampled subtensors obtained by concatenating $\cX_{\cU}^\ell$ with $\cX_{\cC'}^\ell$ and $\cX_{\cR'}^\ell$, respectively.
Importantly, \eqref{eq:low rank update} serves only as a conceptual characterization of the recovered tensor. The tensor is never explicitly formed during the iterations. Instead, \RITCURTC{} stores and updates only the three cross blocks, thereby preserving the memory and computational advantages of the t-CUR representation.

\subsection{Initialization and Stopping Rule}
All three block variables are initialized to zero. Consequently, the initial residual on each observed block is simply $\cP_{\Omega_b}\cY_b$, which is used to compute the initial Welsch correction and the corresponding outlier estimate $\cS_b^0$ according to \eqref{eq:block-welsch-correction}.
We then define
\begin{equation}
    \cP_{\Omega_b}(\widetilde{\cX}_b^0) =\cP_{\Omega_b}(\cY_b-\cS_b^0)
    \quad \textnormal{ and }\quad
    \cP_{\Omega_b^c}(\widetilde{\cX}_b^0)=\bm{0}, \qquad \forall b\in\fB.
\label{eq:robust-initialization}
\end{equation}
In other words, the observed entries are initialized using the outlier-corrected observations, while all unobserved entries are set to zero. The resulting blocks are then projected onto the low-tubal-rank manifold using the same blockwise projection procedure described in \eqref{eq:current-core-tsvd}--\eqref{eq:project blocks}:
\begin{gather} \label{eq:init X_0}
    \cX_{\cU}^0
    =\textnormal{t-SVD}_r(\widetilde{\cX}_{\cU}^0)
    = \cW^0\tprod\boldsymbol{\Theta}^0\tprod(\cV^0)^\top,\cr
    \cX_{\cR'}^0 = \cW^0\tprod(\cW^0)^\top\tprod \widetilde{\cX}_{\cR'}^0 \quad\textnormal{ and }\quad
    \cX_{\cC'}^0 =
    \widetilde{\cX}_{\cC'}^0\tprod\cV^0\tprod(\cV^0)^\top.
\end{gather}

For the stopping criterion, the algorithm may monitor the relative change of the block variables:
\begin{equation*}
    d_\ell=
    \frac{\sum_{b\in\fB}\|\cX_b^\ell-\cX_b^{\ell-1}\|_{\fro}^2}
       {\sum_{b\in\fB}\|\cX_b^{\ell-1}\|_{\fro}^2},
\end{equation*}
or the normalized observed residual:
\begin{equation*}e_\ell=
    \frac{\sum_{b\in\fB}\|\cP_{\Omega_b}(\cX_b^\ell+\cS_b^\ell-\cY_b)\|_{\fro}^2}
    {\sum_{b\in\fB}\|\cP_{\Omega_b}\cY_b\|_{\fro}^2}.
\end{equation*}
The algorithm terminates when either $d_\ell$ or $e_\ell$ falls below a prescribed tolerance, or when a maximum number of iterations is reached.

Importantly, both the initialization procedure and the stopping criteria rely solely on the three stored blocks. Therefore, their evaluation does not require reconstructing the full tensor and will not incur additional memory or computational complexities.

 \section{Numerical Experiments}
\label{sec:exp}

We evaluate the proposed \RITCURTC\ method using synthetic
low-tubal-rank tensors, a cardiac MRI volume, and a three-dimensional
seismic data set.  The synthetic experiments examine the residual
decay and reconstruction accuracy of the method under different tubal
ranks and corruption magnitudes.  The real-data experiments investigate
its robustness to increasing corruption levels and assess the effect of
the observation geometry on reconstruction quality.

\smallskip
\noindent\textbf{Compared methods.}
We compare \RITCURTC\ with the nonrobust t-CCS method \ITCURTC\
\cite{su2024tccs} and the robust full-tensor iterative
hard-thresholding method \RTIHT\ \cite{jiang2019robust}.  Comparing
\RITCURTC\ with \ITCURTC\ under identical t-CCS observations measures
the improvement obtained by incorporating a robust correction into the
t-CUR completion framework.  Applying \RTIHT\ directly to the same
t-CCS observations examines whether a generic robust tensor-completion
method can handle the structured sampling pattern without explicitly
exploiting the sampled cross.  We further apply \RTIHT\ under uniform
entry sampling to illustrate the influence of the observation geometry
and to provide a conventional uniform-sampling benchmark for
\RITCURTC.

Because their t-CUR representations rely on selected horizontal and
lateral subtensors, \RITCURTC\ and \ITCURTC\ are applied only under
t-CCS.  In each comparison between t-CCS and uniform sampling, the
uniform mask contains the same realized number of observations as the
corresponding t-CCS mask.  The two sampling schemes also use the same
number of corrupted entries and the same collection of corruption
amplitudes.  Within each trial, all methods are evaluated using the
same ground-truth tensor and random seed.

\smallskip
\noindent\textbf{Evaluation metrics.}
We use different accuracy measures for the synthetic and real-data
experiments.  For the synthetic tensors, reconstruction accuracy is
measured by the relative error (RelErr) between the recovered tensor ${\cX}_{\textrm{output}}\in\mathbb{R}^{n_1\times n_2\times n_3}$ and its ground truth $\cX^\star\in\mathbb{R}^{n_1\times n_2\times n_3}$
\begin{equation}
    \operatorname{RelErr}
    =
    \frac{
        \|{\cX}_{\textrm{output}}-\cX^\star\|_\fro
    }{
        \|\cX^\star\|_\fro
    }.
    \label{eq:synthetic-relative-error}
\end{equation}
For the MRI and seismic data, reconstruction quality is measured by the
peak signal-to-noise ratio (PSNR)
\begin{equation}
    \operatorname{PSNR}
    =
    10\log_{10}
    \left(
        \frac{
            n_1n_2n_3
            \|\cX^\star\|_\infty^2
        }{
            \|{\cX}_{\textrm{output}}-\cX^\star\|_\fro^2
        }
    \right),
    \label{eq:psnr}
\end{equation}
where \(\|\cX^\star\|_\infty\) denotes the largest absolute entry
of $\cX^\star$.

For the synthetic experiments, we additionally evaluate how well the
estimated corruption magnitudes distinguish corrupted observations from
clean ones. Since the Welsch correction is generally dense, it does not
directly provide a sparse support estimate. For evaluation only, we
therefore define $\widetilde{\Gamma}$ as the set of the
$|\Gamma^\star|$ largest-magnitude entries of
$\widetilde{\cS}_{\Omega}$, where $\Gamma^\star$ denotes the true
corruption support.
Support recovery is measured by the $F_1$ score, defined as the
harmonic mean of precision and recall:
\begin{equation}
    F_1
    =
    \frac{
        2\,\operatorname{precision}\,
        \operatorname{recall}
    }{
        \operatorname{precision}
        +
        \operatorname{recall}
    },
    \label{eq:f1-score}
\end{equation}
where
\[
    \operatorname{precision}
    =
    \frac{\mathrm{TP}}{\mathrm{TP}+\mathrm{FP}},
    \qquad
    \operatorname{recall}
    =
    \frac{\mathrm{TP}}{\mathrm{TP}+\mathrm{FN}}.
\]
Here, $\mathrm{TP}$, $\mathrm{FP}$, and $\mathrm{FN}$ denote the
numbers of true positives, false positives, and false negatives,
respectively, obtained by comparing the estimated support
$\widetilde{\Gamma}$ with the ground-truth support $\Gamma^\star$.
The cardinality $|\Gamma^\star|$ is used only for this diagnostic
post-processing step and is not supplied to the reconstruction
algorithm.

\smallskip
\noindent\textbf{Implementation details.}
All experiments are implemented in Python using double-precision
arithmetic. The t-product, t-SVD, and tubal-rank projections are
computed through the discrete Fourier transform along the third tensor
mode. The experiments were performed on a laptop equipped with an
11th-generation Intel Core i7-11800H processor at 2.30~GHz and
16~GB of RAM.

\subsection{Synthetic Experiments}
\label{subsec:synthetic-experiments}

We first examine \RITCURTC\ on synthetic tensors with a prescribed
low-tubal-rank structure.  Following the matrix experiments in
\cite{cai2024rccs}, we generate
\(
    \bm{\mathcal A}\in\mathbb R^{n_1\times r\times n_3},
    \bm{\mathcal B}\in\mathbb R^{r\times n_2\times n_3},
\)
whose entries are independently drawn from the standard Gaussian
distribution, and define
\(
    \cX^\star=\bm{\mathcal A}*\bm{\mathcal B}.
\)
Thus,
\(\operatorname{rank}_t(\cX^\star)\leq r\).
To make the corruption levels comparable across tensor dimensions and
tubal ranks, we normalize the tensor to have unit root-mean-square
amplitude:
\(
    \frac{\|\cX^\star\|_\fro^2}
         {n_1n_2n_3}
    =1.
\)

\smallskip
\noindent\textbf{Sampling and corruption models.}
We generate a tensor cross-concentrated sampling mask with cross
fraction \(\delta\) and within-cross sampling rate \(p\).  When the same
cross fraction is used in the first two modes, the nominal overall
observation fraction is
\(
    \rho(p,\delta) =  2\delta p-(\delta p)^2.
\)

Let \(\Omega\) denote the resulting observation set.  We select
\(\lceil\alpha|\Omega|\rceil\) observed entries uniformly at random and
denote the corruption support by
\(\Gamma^\star\subseteq\Omega\).  On this support, the corruption
entries are generated independently according to
\begin{equation}
    [\cS^\star]_{ijk}
    \sim
    \operatorname{Unif}[-c\mu_X,c\mu_X],
    \qquad
    \mu_X
    =
    \frac{\|\cX^\star\|_1}{n_1n_2n_3},
    \qquad
    (i,j,k)\in\Gamma^\star.
    \label{eq:synthetic-corruption-model}
\end{equation}
The observed tensor is therefore
\(
    \cY_\Omega
    =
    \mathcal P_\Omega
    \left(
        \cX^\star+\cS^\star
    \right).
\)

\smallskip \noindent\textbf{Algorithmic settings.}
The same Welsch-based implementation of \RITCURTC\ is used in all
synthetic and real-data experiments.  For the synthetic experiments,
the three block step sizes are set to
\(
    \eta_{\mathrm R}
    =
    \eta_{\mathrm C}
    =
    \eta_{\mathrm U}
    =1.
\)
The Welsch scale is estimated separately for each observed block using
the adaptive quantile rule described in Section~\ref{subsection:outlierDet}.
The algorithm is run for at least \(20\) and at most \(100\) iterations.
After the first \(20\) iterations, it is terminated once the relative
change between successive iterates remains below \(5\times10^{-5}\) for
five consecutive iterations.

\smallskip \noindent\textbf{Observed-residual decay.}
We consider tensors of size \(120\times120\times8\) and fix the corruption rate   $\alpha=0.20$,   and the sampling parameters $p=\delta=0.40$,  which gives the nominal observation fraction
\(\rho=0.2944\).  We conduct two sets of experiments.  First, we fix the
corruption-magnitude factor at \(c=10\) and vary the tubal rank over
\(
    r\in\{2,3,5\}.
\)
Second, we fix \(r=3\) and vary the corruption-magnitude factor over
\(    c\in\{10,50,200\}.
\)
Each configuration is repeated over ten independent trials.

At iteration \(\ell\), we record the squared normalized residual on the
observed entries:
\begin{equation}
    e_{\ell}
    =
    \frac{
        \|
            \mathcal P_\Omega
            (
                 {\cX}^{\,\ell}
                +
                 {\cS}^{\,\ell}
                -
                \cY_\Omega
            )
        \|_\fro^2
    }{
        \|\cY_\Omega\|_\fro^2
    },
    \qquad
    {\cX}^{\,\ell}
    =
    \cC^{\ell}*
    (\cU^{\ell})^\dagger*
    \cR^{\ell}.
    \label{eq:synthetic-observed-residual}
\end{equation}
This quantity can be computed entirely from the observed data and is
used only as a convergence diagnostic.  In
Fig.~\ref{fig:synthetic-convergence}, the solid curves represent the
median trajectories over the ten trials, while the shaded regions
represent the corresponding interquartile ranges.  When a trial
terminates before iteration \(100\), its final residual is carried
forward in computing the aggregate trajectory.

\begin{figure}[ht]
    \centering
    \begin{subfigure}[t]{0.485\textwidth}
        \centering
        \includegraphics[width=\linewidth]{
            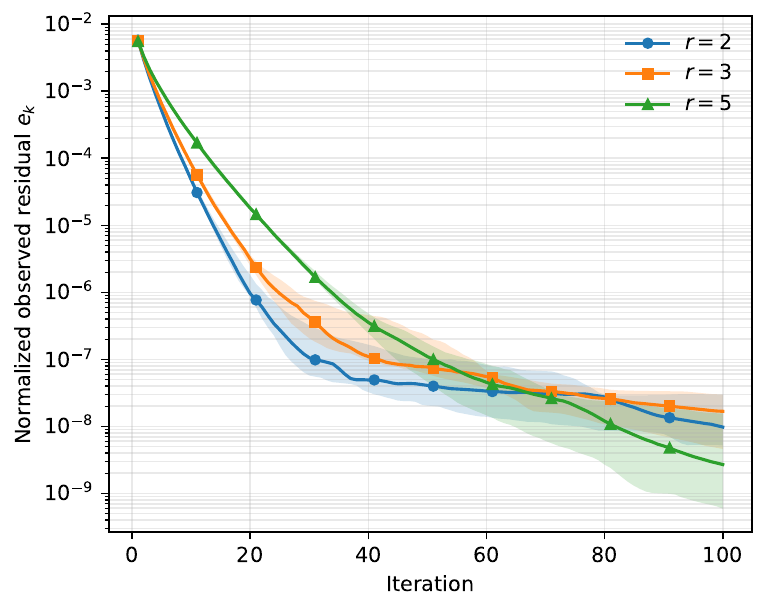}
        \caption{Fixed \(c=10\) and varying tubal rank.}
        \label{fig:synthetic-convergence-rank}
    \end{subfigure}
    \hfill
    \begin{subfigure}[t]{0.485\textwidth}
        \centering
        \includegraphics[width=\linewidth]{
            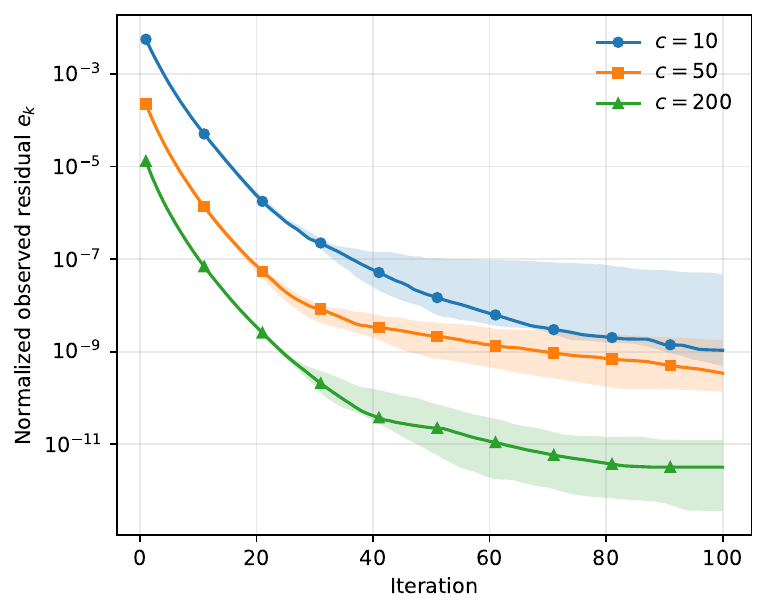}
        \caption{Fixed \(r=3\) and varying corruption magnitude.}
        \label{fig:synthetic-convergence-magnitude}
    \end{subfigure}
    \caption{
        Observed-residual decay of the Welsch-based \RITCURTC\ method
        on synthetic low-tubal-rank tensors.  The curves and shaded
        regions show the medians and interquartile ranges, respectively,
        over ten independent trials.
    }
    \label{fig:synthetic-convergence}
\end{figure}

\begin{table}[ht]
    \centering
    \caption{
        Synthetic reconstruction and corruption-separation results.
        Each row summarizes ten independent trials under
        \(\alpha=0.20\), \(p=\delta=0.40\), and tensor size
        \(120\times120\times8\).
    }
    \label{tab:syn_welsch_con}
    \small

\begin{tabular}{cccccc}
\toprule
Experiment & Parameter & Median RelErr & Iter. & Median Top-$k$ $F_1$ & Min. Top-$k$ $F_1$ \\
\midrule
\multirow{3}{*}{Corruption}
& $c=10$  & $3.44\times10^{-2}$ & 100 & 0.9802 & 0.9710 \\
& $c=50$  & $2.46\times10^{-2}$ & 100 & 0.9948 & 0.9928 \\
& $c=200$ & $2.20\times10^{-2}$ & 93  & 0.9982 & 0.9973 \\
\midrule
\multirow{3}{*}{Tubal rank}
& $r=2$ & $3.82\times10^{-2}$ & 100 & 0.9799 & 0.9687 \\
& $r=3$ & $4.87\times10^{-2}$ & 100 & 0.9797 & 0.9680 \\
& $r=5$ & $3.07\times10^{-2}$ & 100 & 0.9819 & 0.9742 \\
\bottomrule
\end{tabular}
\end{table}

Figure~\ref{fig:synthetic-convergence} and
Table~\ref{tab:syn_welsch_con} show that \RITCURTC\ consistently reduces
the observed residual for all tested tubal ranks and corruption
magnitudes.  The final relative errors remain between approximately
\(2\times10^{-2}\) and \(5\times10^{-2}\), indicating that the observed
data are fitted increasingly well even though the reconstruction error
does not vanish.  Moreover, the results for \(r\in\{2,3,5\}\) exhibit no
clear monotone dependence on the tubal rank over the tested range.  Since
most trials reach the iteration limit, the curves should be interpreted
as evidence of stable residual reduction rather than strict convergence
under the prescribed stopping criterion.

A clearer trend is observed as the corruption magnitude increases.
When \(c\) increases from \(10\) to \(200\), the median relative error
decreases from \(3.44\times10^{-2}\) to \(2.20\times10^{-2}\), while the
median top-\(|\Gamma^\star|\) \(F_1\)-score increases from \(0.9802\) to
\(0.9982\).  At a fixed corruption fraction, larger outliers are more
easily distinguished from the low-tubal-rank component, allowing the
Welsch correction to suppress them more effectively.  Thus, the improved
recovery at larger \(c\) reflects increased separation between the clean
signal and the corruptions.

\subsection{Real-Data Experiments}
\label{subsec:real-data-experiments}

We evaluate \RITCURTC\ on two real datasets: a cardiac MRI volume
from the \emph{Task02\_Heart} dataset of the Medical Segmentation
Decathlon\footnote{The Medical Segmentation Decathlon data are available at
\url{http://medicaldecathlon.com/dataaws}.}
\cite{simpson2019medical,antonelli2022medical}, and the F3 Demo
three-dimensional seismic dataset provided by Terranubis\footnote{The F3 Demo seismic dataset is available at
\url{https://terranubis.com/datainfo/F3-Demo-2020}.}.
Because the Fourier transform underlying the t-product is applied along
the third mode, the ordering of the tensor modes forms part of the
reconstruction model.  Table~\ref{tab:datasets} summarizes the tensor
representation and target rank used for each dataset.

\begin{table}[ht]
    \centering
    \caption{Real datasets used in the experiments.}
    \label{tab:datasets}
    \small
    \begin{tabular}{lccc}
        \toprule
        Dataset
        & Tensor size
        & Mode ordering
        & Target rank \(r\) \\
        \midrule
        MRI
        & \(320\times320\times110\)
        & \(x\)-coordinate \(\times\) \(y\)-coordinate
          \(\times\) axial slice
        & \(35\) \\
        Seismic
        & \(51\times191\times146\)
        & time sample \(\times\) inline \(\times\) crossline
        & \(3\) \\
        \bottomrule
    \end{tabular}
\end{table}

\subsubsection{MRI Reconstruction}
\label{subsec:mri-experiments}

We examine the proposed method on an MRI data of size
\(    320\times320\times110.
\)
The volume is globally normalized to the interval $[0,1]$, and all
t-products and t-SVDs use the discrete Fourier transform along the
third mode. The target tubal rank is fixed at $r=35$.

\smallskip \noindent\textbf{Experimental setup.}
For the quantitative experiment, we follow the nonzero-support MRI
sampling protocol used in \cite{su2024tccs}. For each
trial, index sets
\(
    I\subset[n_1],
    J\subset[n_2]
\)
are selected according to the cross fraction $\delta$. Among the entries
in the resulting horizontal and lateral subtensors, only locations for
which
\(    [\cX^\star]_{ijk}\neq0
\)
are eligible for observation. At most
\(
    \left\lceil
        \rho n_1n_2n_3
    \right\rceil
\)
eligible entries are sampled, where the nominal observation fraction is
fixed at $\rho=0.30$.

Because the selected cross may contain fewer than
$0.3 n_1n_2n_3$ nonzero entries, the realized observation fraction can
be smaller than the nominal value.
The selected values of $\delta$ and the corresponding realized observation rates are listed in Table~\ref{tab:obsrate}.
\begin{table}[ht]
\centering
\caption{Realized observation rates (mean $\pm$ standard deviation over three reporting seeds).}
\label{tab:obsrate}
\begin{tabular}{c|ccc}
\toprule
Target $\delta$ & 0.23 & 0.25 & 0.27 \\
\midrule
Realized observation rate
& $0.2688 \pm 0.0035$
& $0.2861 \pm 0.0053$
& $0.2980 \pm 0.0021$ \\
\bottomrule
\end{tabular}
\end{table}
In addition, we choose
\(
    \alpha\in\{0,0.05,0.10,0.15\}.
\)
Given an observation set $\Omega$, exactly
\(
    \operatorname{round}(\alpha|\Omega|)
\)
observed entries are chosen uniformly at random as the corruption
support. The nonzero corruption values are sampled independently from
\begin{equation}
    [\cS^\star]_{ijk}
    \sim
    \operatorname{Unif}[-c\mu_X,c\mu_X],
    \qquad
    c=3,
    \qquad
    \mu_X
    =
    \frac{1}{n_1n_2n_3}
    \|\cX^\star\|_1.
    \label{eq:mri-corruption-model}
\end{equation}
Within each paired experiment, the t-CCS and uniform-sampling models
use the same number of observed entries, the same number of corrupted
entries, and the same collection of corruption amplitudes.

We compare the Welsch \RITCURTC\ and nonrobust \ITCURTC\ under
the same t-CCS mask and corrupted observations. We also apply the same
robust full-tensor IHT implementation under t-CCS and matched uniform
entry sampling. The \RITCURTC\ and \ITCURTC\ methods are run for
$20$ iterations, whereas the two \RTIHT\ baselines are run for
$40$ iterations. All entries in Table~\ref{tab:mri-results} are means
and standard deviations over the three random seeds.

\begin{table*}[ht]
    \centering
    \caption{
        MRI reconstruction PSNR in dB under the nonzero-support t-CCS
        protocol and matched uniform sampling. Entries are mean
        $\pm$ standard deviation over three random seeds. The nominal
        observation fraction is $\rho=0.30$, the target tubal rank is
        $r=35$, and the corruption-magnitude factor is $c=3$. The best
        mean PSNR in each row is shown in bold.
    }
    \label{tab:mri-results}
    \small
    \setlength{\tabcolsep}{5pt}
    \renewcommand{\arraystretch}{1.08}
    \begin{tabular}{cccccc}
        \toprule
        $\rho$
        & $\alpha$
        & \makecell{\RITCURTC\\(t-CCS)}
        & \makecell{\ITCURTC\\(t-CCS)}
        & \makecell{\RTIHT\\(t-CCS)}
        & \makecell{\RTIHT\\(uniform)} \\
        \midrule

        \multirow{4}{*}{$0.23$}
        & $0.00$
        & $\mathbf{30.30\pm0.60}$
        & $\mathbf{30.30\pm0.60}$
        & $19.80\pm0.19$
        & $21.94\pm0.06$ \\

        & $0.05$
        & $\mathbf{29.32\pm0.53}$
        & $28.78\pm0.39$
        & $19.54\pm0.17$
        & $21.52\pm0.05$ \\

        & $0.10$
        & $\mathbf{28.58\pm0.44}$
        & $27.63\pm0.27$
        & $19.48\pm0.16$
        & $21.32\pm0.05$ \\

        & $0.15$
        & $\mathbf{27.81\pm0.35}$
        & $26.75\pm0.21$
        & $19.44\pm0.15$
        & $21.18\pm0.05$ \\

        \midrule

        \multirow{4}{*}{$0.25$}
        & $0.00$
        & $\mathbf{31.48\pm0.76}$
        & $\mathbf{31.48\pm0.76}$
        & $19.95\pm0.03$
        & $22.27\pm0.08$ \\

        & $0.05$
        & $\mathbf{30.16\pm0.75}$
        & $29.51\pm0.54$
        & $19.69\pm0.04$
        & $21.80\pm0.07$ \\

        & $0.10$
        & $\mathbf{29.39\pm0.68}$
        & $28.23\pm0.44$
        & $19.61\pm0.05$
        & $21.58\pm0.07$ \\

        & $0.15$
        & $\mathbf{28.58\pm0.62}$
        & $27.28\pm0.36$
        & $19.57\pm0.05$
        & $21.43\pm0.07$ \\

        \midrule

        \multirow{4}{*}{$0.27$}
        & $0.00$
        & $\mathbf{31.58\pm0.95}$
        & $\mathbf{31.58\pm0.95}$
        & $20.62\pm0.47$
        & $22.48\pm0.04$ \\

        & $0.05$
        & $\mathbf{30.31\pm0.79}$
        & $29.86\pm0.39$
        & $20.30\pm0.41$
        & $21.99\pm0.04$ \\

        & $0.10$
        & $\mathbf{29.57\pm0.64}$
        & $28.58\pm0.29$
        & $20.19\pm0.39$
        & $21.76\pm0.03$ \\

        & $0.15$
        & $\mathbf{28.83\pm0.50}$
        & $27.58\pm0.23$
        & $20.13\pm0.37$
        & $21.59\pm0.03$ \\

        \bottomrule
    \end{tabular}
\end{table*}

\smallskip \noindent\textbf{Quantitative results.}
Table~\ref{tab:mri-results} shows that \RITCURTC\ matches \ITCURTC\
in the clean case and consistently achieves higher PSNR when
corruptions are present. Moreover, the performance gap generally
increases with the corruption fraction, demonstrating that the Welsch
correction improves robustness to sparse gross errors without degrading
clean-data reconstruction.

The \RTIHT\ baseline performs better under uniform sampling than under
t-CCS, indicating that it does not exploit the cross-concentrated sampling structure as effectively as \RITCURTC. Among the methods using
the same t-CCS observations, \RITCURTC\ provides the best reconstruction
quality throughout the corrupted cases.

\smallskip \noindent\textbf{Value-independent representative reconstruction.}
To additionally illustrate performance under value-independent
sampling, we consider a representative realization generated using the
\texttt{all\_entries} protocol. In this experiment, every location in
the selected tensor cross is eligible for observation, independently
of its ground-truth intensity. The parameters are $\rho=0.30$, $\delta=0.25$, $\alpha=0.10$, and $r=35$,
with random seed $20260731$. The resulting full-volume PSNR values are
\[
\begin{aligned}
    \RITCURTC\text{ (t-CCS)}
        &:\ 24.3560\ \mathrm{dB},\\
    \ITCURTC\text{ (t-CCS)}
        &:\ 23.7057\ \mathrm{dB},\\
    \RIHT\text{ (t-CCS)}
        &:\ 19.2252\ \mathrm{dB},\\
    \RIHT\text{ (uniform)}
        &:\ 21.8052\ \mathrm{dB}.
\end{aligned}
\]
The corresponding common-test PSNR values, evaluated only on entries
unobserved under both sampling models, are $23.1636$, $22.5163$,
$17.7483$, and $20.2146$ dB, respectively.
Figure~\ref{fig:mri-representative} displays slices $27$, $54$, and
$82$. For each row, all reconstruction methods use the same intensity
range computed from the corresponding ground-truth slice.

\begin{figure}[!htbp]
    \centering
    \setlength{\tabcolsep}{1.5pt}
    \renewcommand{\arraystretch}{0.35}

    \begin{tabular}{@{}ccccc@{}}
      {Ground truth}
        &
        \makecell{\RITCURTC\\ (t-CCS)}
        &
        \makecell{\ITCURTC\\ (t-CCS)}
        &
        \makecell{\RTIHT\\ (t-CCS)}
        &
        \makecell{\RTIHT\\(uniform)}
        \\

        \includegraphics[width=0.192\textwidth]{
            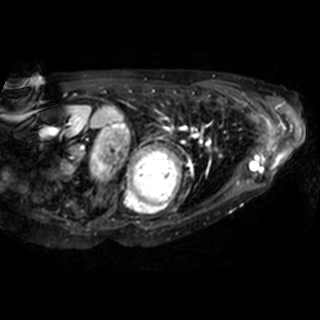}
        &
        \includegraphics[width=0.192\textwidth]{
            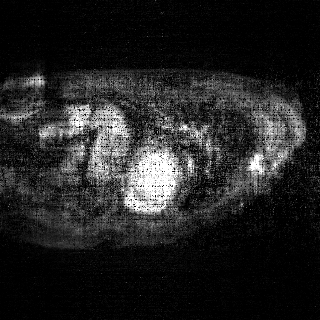}
        &
        \includegraphics[width=0.192\textwidth]{
            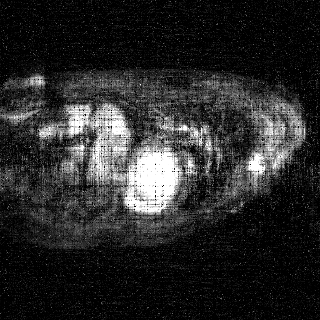}
        &
        \includegraphics[width=0.192\textwidth]{
            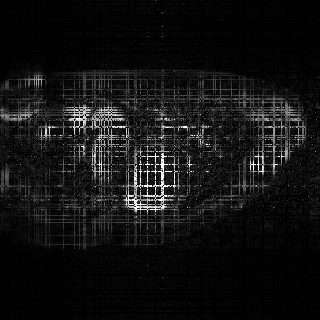}
        &
        \includegraphics[width=0.192\textwidth]{
            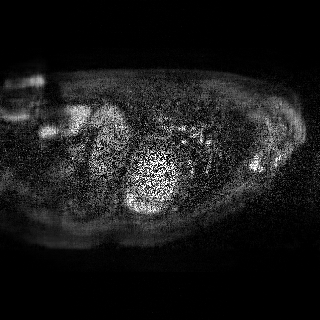}
        \\

        \includegraphics[width=0.192\textwidth]{
            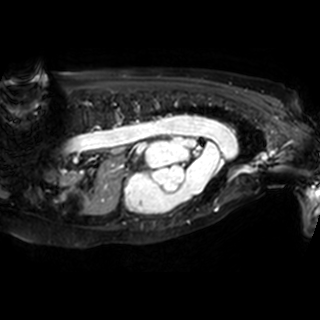}
        &
        \includegraphics[width=0.192\textwidth]{
            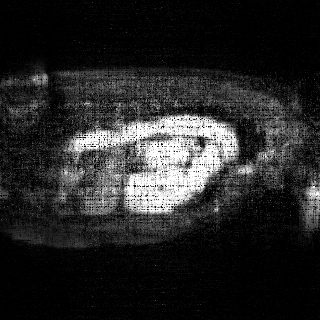}
        &
        \includegraphics[width=0.192\textwidth]{
            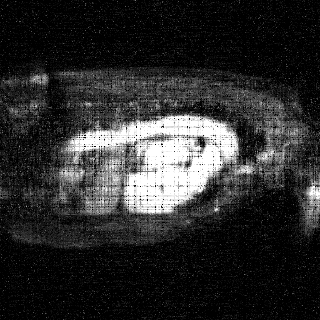}
        &
        \includegraphics[width=0.192\textwidth]{
            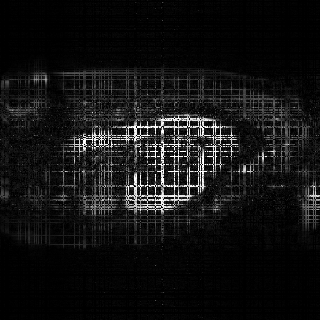}
        &
        \includegraphics[width=0.192\textwidth]{
            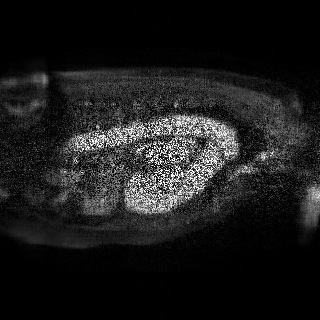}
        \\

        \includegraphics[width=0.192\textwidth]{
            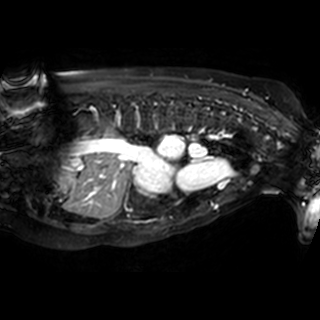}
        &
        \includegraphics[width=0.192\textwidth]{
            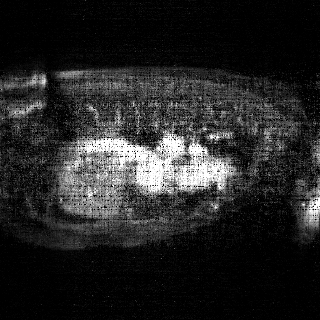}
        &
        \includegraphics[width=0.192\textwidth]{
            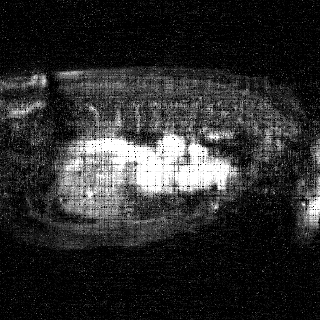}
        &
        \includegraphics[width=0.192\textwidth]{
            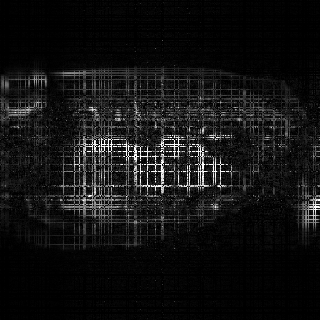}
        &
        \includegraphics[width=0.192\textwidth]{
            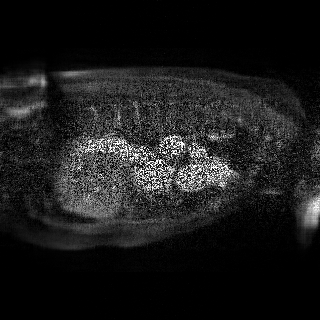}
    \end{tabular}

    \caption{
        Representative MRI reconstructions under value-independent
        sampling with $\rho=0.30$, $\delta=0.25$, $\alpha=0.10$,
        and $r=35$. From top to bottom, the rows show axial slices
        $27$, $54$, and $82$. All panels in the same row use a common
        display range determined from the ground-truth slice. This
        representative experiment uses the \texttt{all\_entries}
        protocol and is therefore distinct from the nonzero-support
        quantitative experiment reported in Table~\ref{tab:mri-results}.
    }
    \label{fig:mri-representative}
\end{figure}

\subsubsection{Seismic Data Recovery}
\label{subsec:seismic-experiment}

For the seismic tensor described above, we set the target tubal rank to
$r=3$ and fix the t-CCS cross-size parameter at $\delta=0.17$. After
integer rounding, the selected index sets satisfy
\(
    |I|=9,
    |J|=32.
\)
We consider the overall observation fractions
\(
    \rho\in\{0.20,0.25,0.30\}
\)
and outlier fractions
\(
    \alpha\in\{0,0.05,0.10,0.15\}.
\)

For each target observation fraction $\rho$, we choose the sampling
probability $p$ so that the expected fraction of observed tensor entries
equals $\rho$. Let
\(
    \delta_1=\frac{|I|}{51},
    \delta_2=\frac{|J|}{191}
\)
denote the relative sizes of the selected horizontal and lateral
subtensors. Since these subtensors overlap on
$\cX^\star(I,J,:)$, their observation masks are combined by
logical union, and hence
\begin{equation}
    \rho
    =
    p\delta_1+p\delta_2-p^2\delta_1\delta_2.
    \label{eq:seismic-sampling-rate}
\end{equation}
Solving this equation gives
\(
    p\approx0.614,  0.779,  0.950
\)
for $\rho=0.20$, $0.25$, and $0.30$, respectively.
For each realization, $\lceil\alpha|\Omega|\rceil$ observed entries are
selected uniformly at random as the outlier support. Since the seismic
tensor is normalized to unit peak magnitude, each selected entry is
perturbed by a randomly signed value of magnitude $3$. For the
uniform-sampling baseline, we use the same realized number of
observations, the same number of outliers, and the same collection of
outlier magnitudes as in the corresponding t-CCS experiment.
 We compare \RITCURTC\ and nonrobust \ITCURTC\ under the same t-CCS
observations. We also apply robust full-tensor IHT under t-CCS and under
matched uniform entry sampling. Each reported value is averaged over
three independent random seeds.

\begin{table}[ht]
    \centering
    \caption{
        Seismic reconstruction PSNR in dB for different observation
        fractions $\rho$ and outlier fractions $\alpha$. Each entry is
        averaged over three independent random seeds. The uniform-sampling
        baseline uses the same realized observation and corruption
        cardinalities as the corresponding t-CCS experiment.
    }
    \label{tab:seismic-results}
    \small
    \setlength{\tabcolsep}{5pt}
    \renewcommand{\arraystretch}{1.08}
    \begin{tabular}{cccccc}
        \toprule
        $\rho$
        & $\alpha$
        & \makecell{\RITCURTC\\(t-CCS)}
        & \makecell{\ITCURTC\\(t-CCS)}
        & \makecell{\RTIHT\\(t-CCS)}
        & \makecell{\RTIHT\\(uniform)} \\
        \midrule

        \multirow{4}{*}{$0.20$}
        & $0.00$
        & $\mathbf{27.90}$
        & $\mathbf{27.90}$
        & $2.54$
        & $27.72$ \\

        & $0.05$
        & $\mathbf{27.38}$
        & $8.69$
        & $2.49$
        & $26.75$ \\

        & $0.10$
        & $\mathbf{27.28}$
        & $5.80$
        & $2.39$
        & $25.34$ \\

        & $0.15$
        & $\mathbf{26.19}$
        & $4.11$
        & $2.25$
        & $23.98$ \\

        \midrule

        \multirow{4}{*}{$0.25$}
        & $0.00$
        & $29.18$
        & $29.18$
        & $2.54$
        & $\mathbf{29.79}$ \\

        & $0.05$
        & $28.49$
        & $10.57$
        & $2.51$
        & $\mathbf{29.57}$ \\

        & $0.10$
        & $28.24$
        & $7.72$
        & $2.43$
        & $\mathbf{29.25}$ \\

        & $0.15$
        & $27.93$
        & $6.02$
        & $2.29$
        & $\mathbf{29.14}$ \\

        \midrule

        \multirow{4}{*}{$0.30$}
        & $0.00$
        & $30.11$
        & $30.11$
        & $2.54$
        & $\mathbf{30.32}$ \\

        & $0.05$
        & $29.16$
        & $11.93$
        & $2.53$
        & $\mathbf{30.20}$ \\

        & $0.10$
        & $28.87$
        & $9.12$
        & $2.43$
        & $\mathbf{29.94}$ \\

        & $0.15$
        & $28.66$
        & $7.42$
        & $2.33$
        & $\mathbf{29.89}$ \\

        \bottomrule
    \end{tabular}
\end{table}

\smallskip \noindent\textbf{Quantitative results.}
Table~\ref{tab:seismic-results} shows that \RITCURTC\ agrees with
\ITCURTC\ in the clean case and is substantially more robust when
outliers are present. In particular, the PSNR of \RITCURTC\ remains
stable as $\alpha$ increases, whereas the performance of nonrobust
\ITCURTC\ deteriorates sharply. The \RTIHT\ baseline performs considerably better under uniform
sampling than under t-CCS, indicating that robustness alone does not
adequately address the cross-concentrated observation geometry.
At $\rho=0.20$, \RITCURTC\ achieves the highest PSNR for every tested
outlier fraction. At the two higher observation fractions, uniform
\RTIHT\ gives slightly higher PSNR, while \RITCURTC\ remains competitive
under the structured t-CCS acquisition model.

\smallskip \noindent\textbf{Representative reconstruction.}
We further present a representative reconstruction with
\(
    \rho=0.25
\)
and
\(
    \alpha=0.10
\),
using random seed $1$. The displayed time-sample indices are
$16$, $27$, and $38$, selected from the high-content portion of the
seismic cube according to the standard deviation of each slice. Within
each row, all methods use the same intensity range determined from the
corresponding ground-truth slice. As shown in
Figure~\ref{fig:seismic-reconstruction}, \RITCURTC\ preserves the main
reflective structures and spatial continuity of the seismic data despite
the sparse gross corruptions. In contrast, nonrobust \ITCURTC\ exhibits
substantial reconstruction artifacts, while uniform-sampling \RTIHT\
produces visually accurate reconstructions but uses a different,
unstructured observation pattern. These results are consistent with the
quantitative comparison in Table~\ref{tab:seismic-results}.

\begin{figure}[ht]
    \centering
    \setlength{\tabcolsep}{2pt}
    \renewcommand{\arraystretch}{0.4}

    \begin{tabular}{@{}cccc@{}}
       {Ground truth}
        &
        \makecell{\RITCURTC\\(t-CCS)}
        &
        \makecell{{\ITCURTC}\\{(t-CCS)}}
        &
        \makecell{{\RTIHT}\\{(uniform)}} \\

        \includegraphics[width=0.24\textwidth]{
            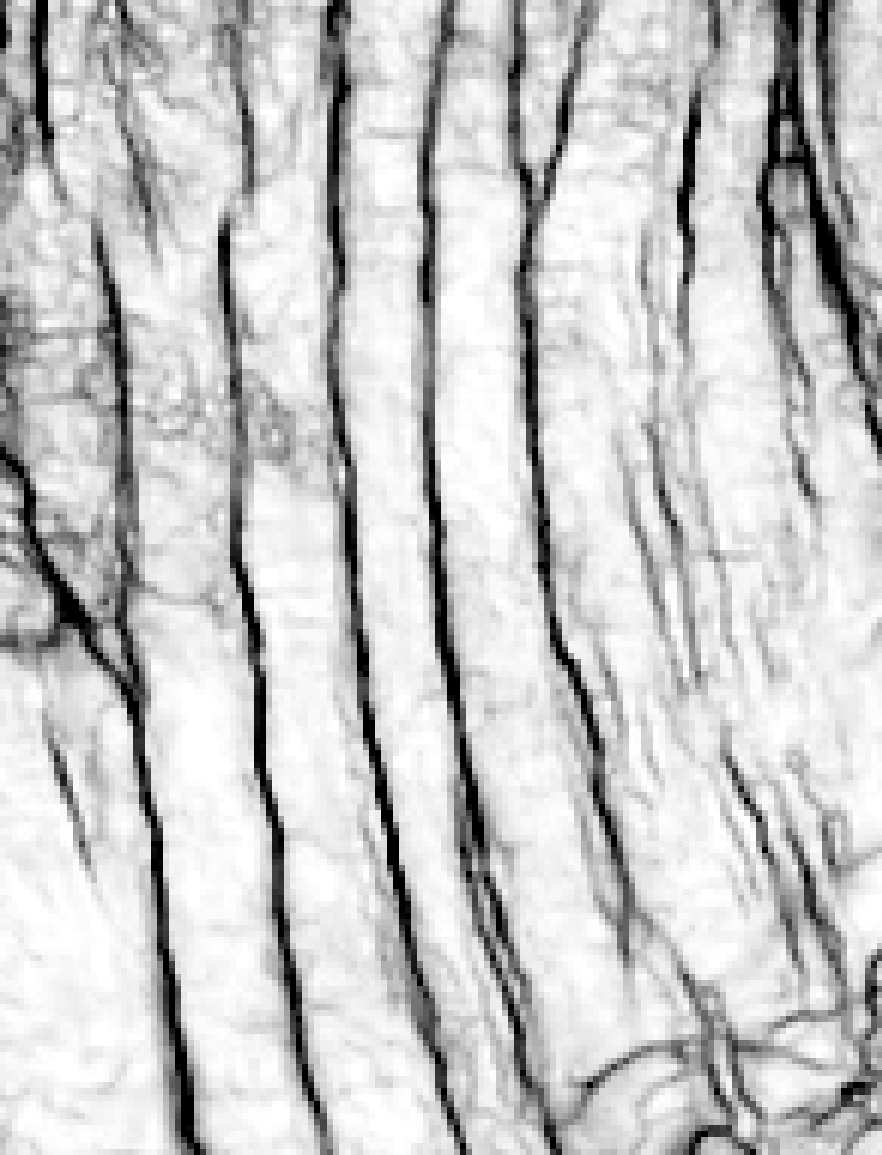}
        &
        \includegraphics[width=0.24\textwidth]{
            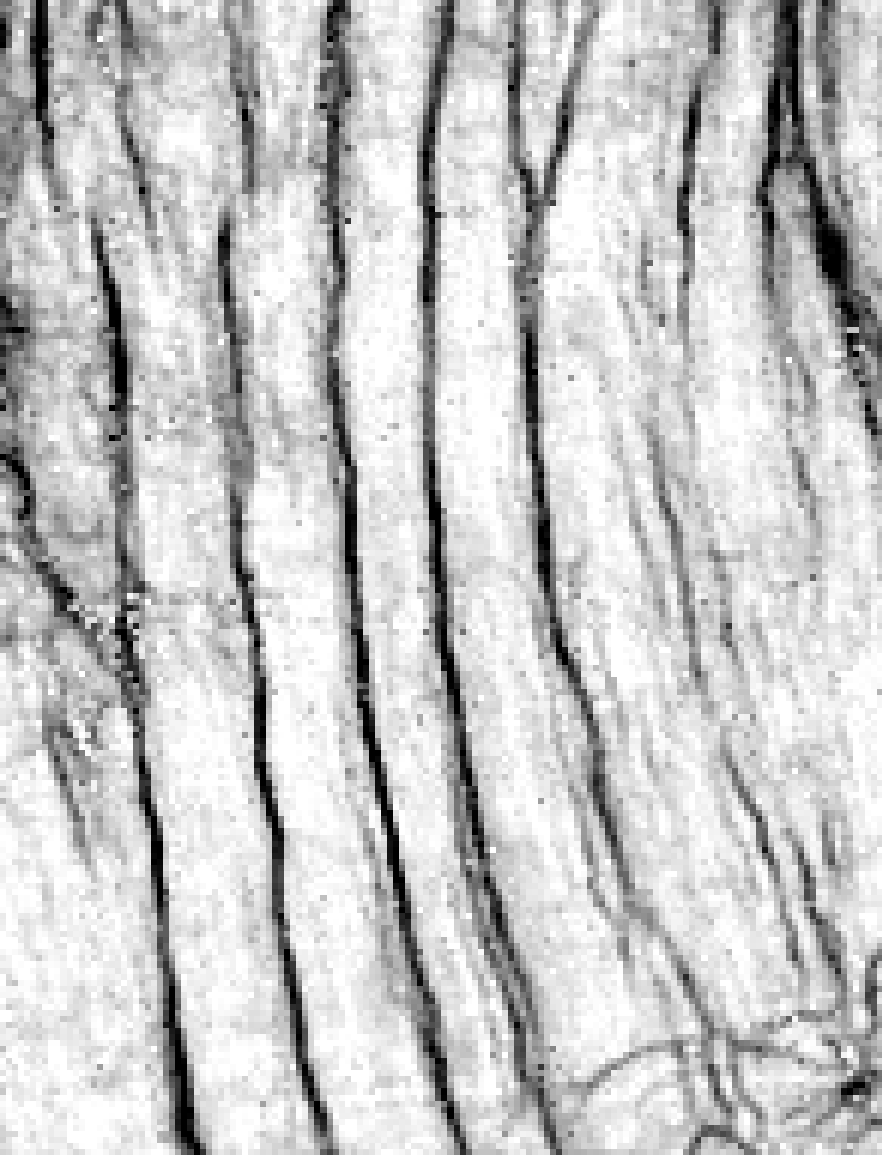}
        &
        \includegraphics[width=0.24\textwidth]{
            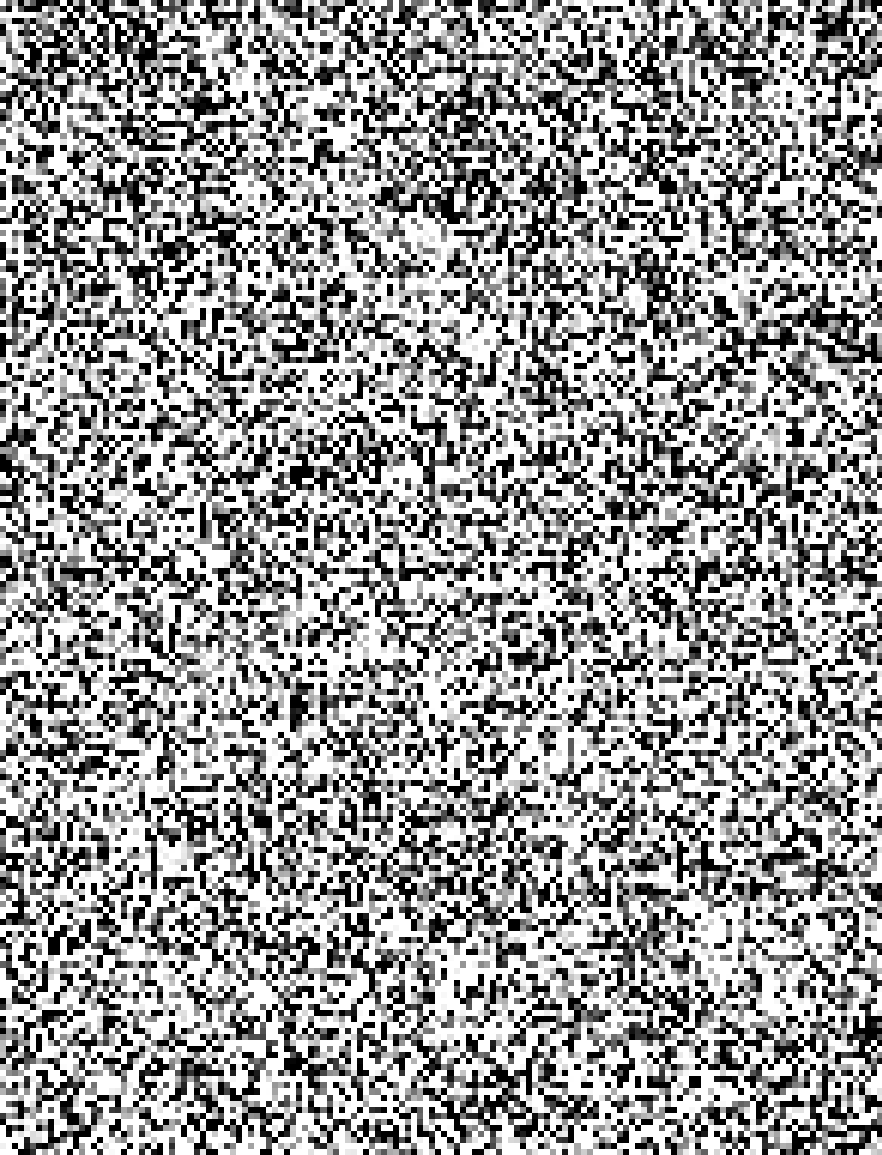}
        &
        \includegraphics[width=0.24\textwidth]{
            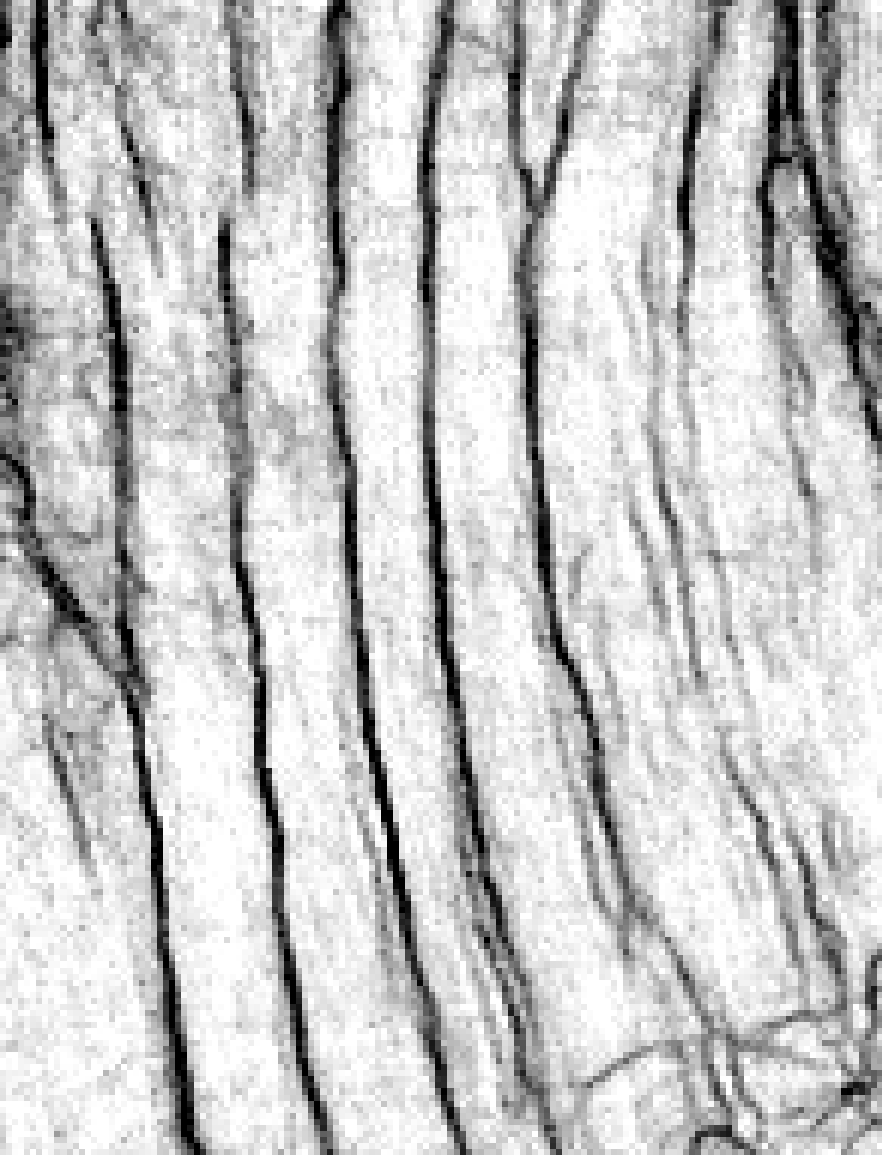} \\

        \includegraphics[width=0.24\textwidth]{
            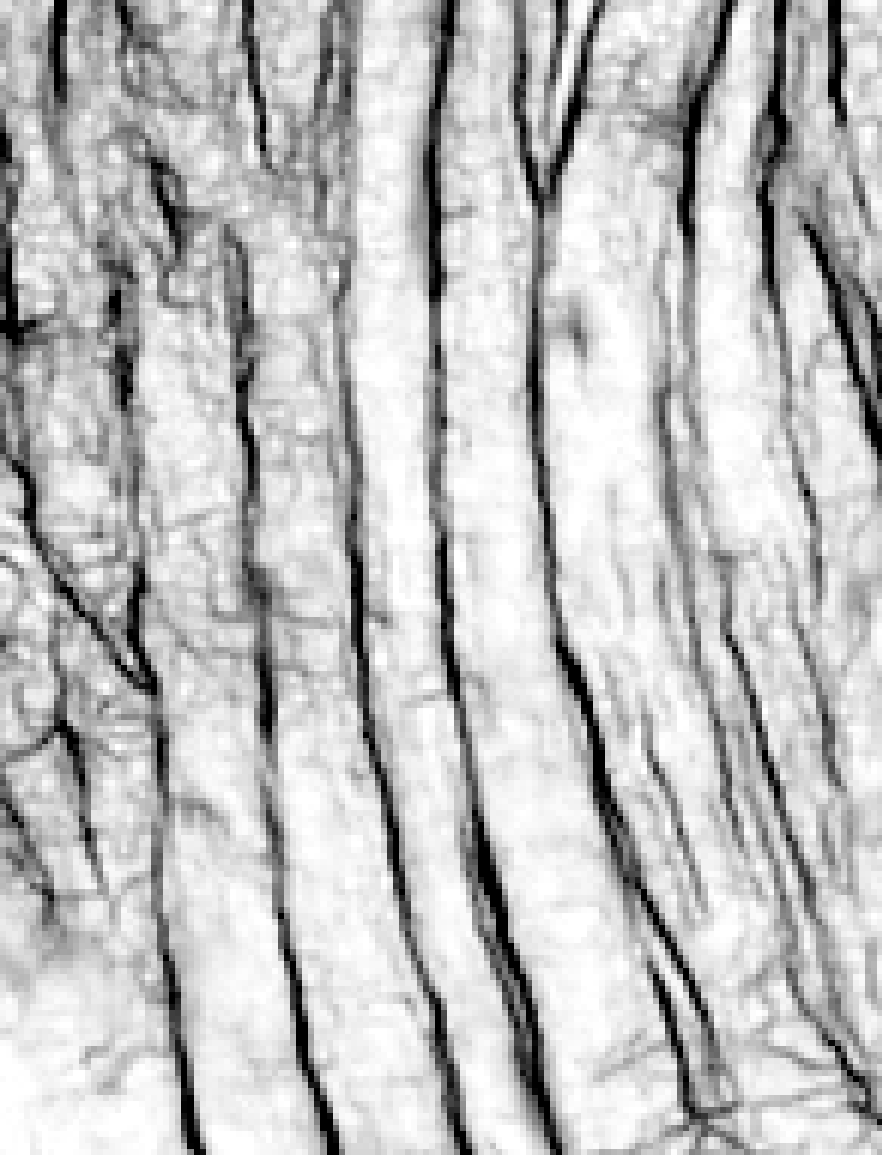}
        &
        \includegraphics[width=0.24\textwidth]{
            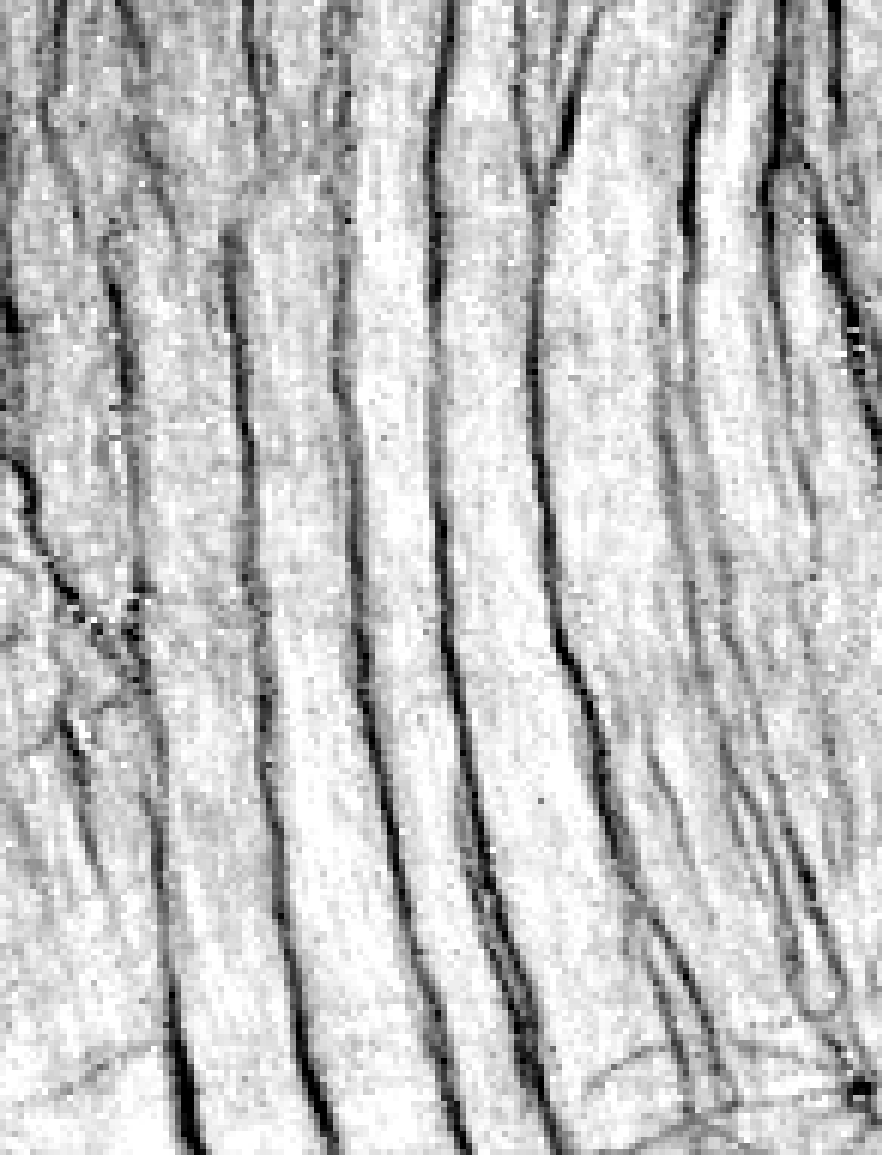}
        &
        \includegraphics[width=0.24\textwidth]{
            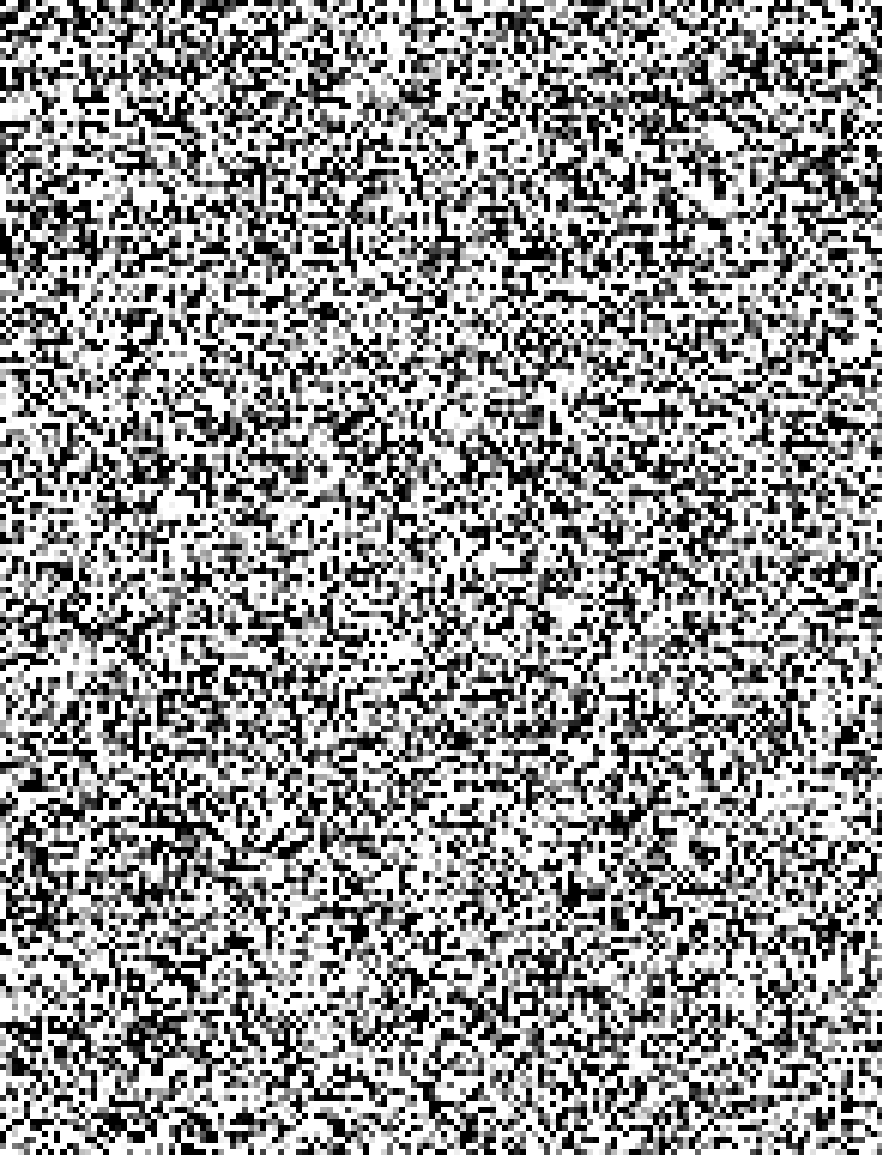}
        &
        \includegraphics[width=0.24\textwidth]{
            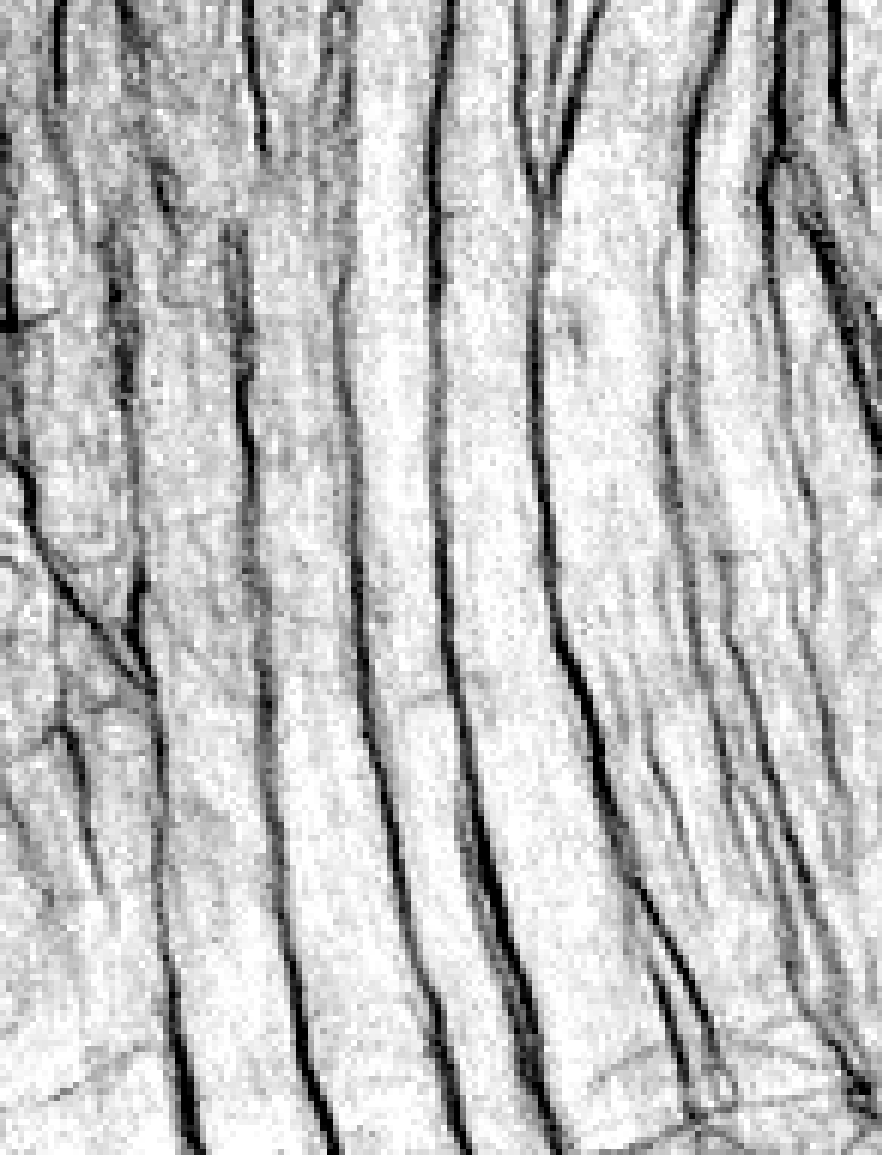} \\

        \includegraphics[width=0.24\textwidth]{
            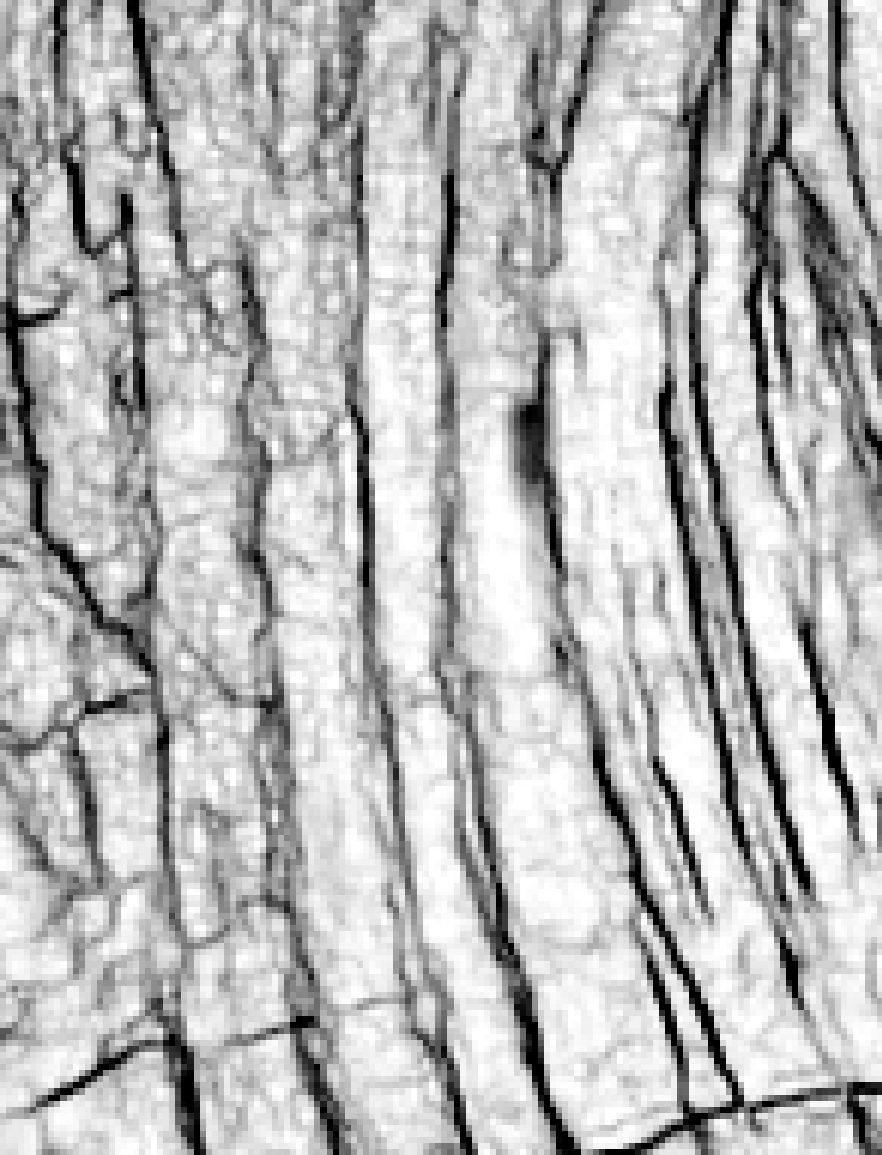}
        &
        \includegraphics[width=0.24\textwidth]{
            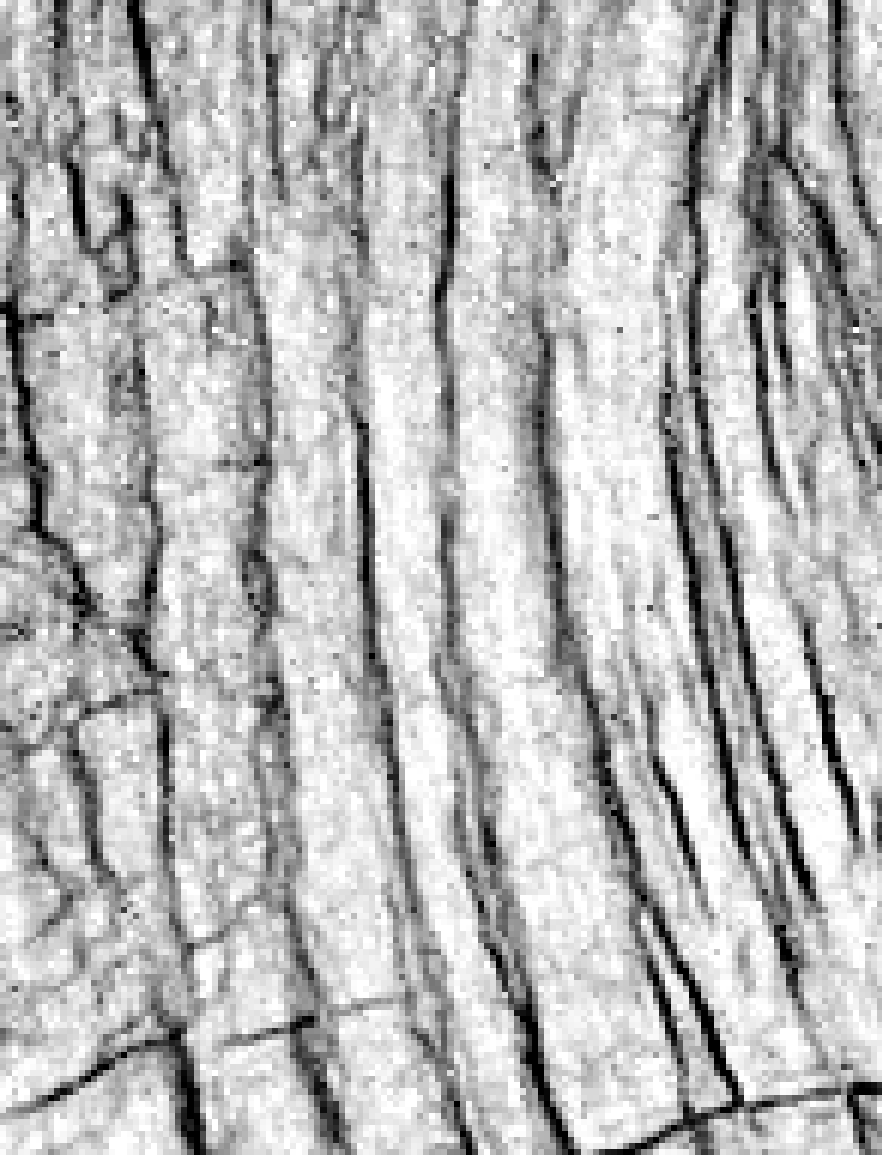}
        &
        \includegraphics[width=0.24\textwidth]{
            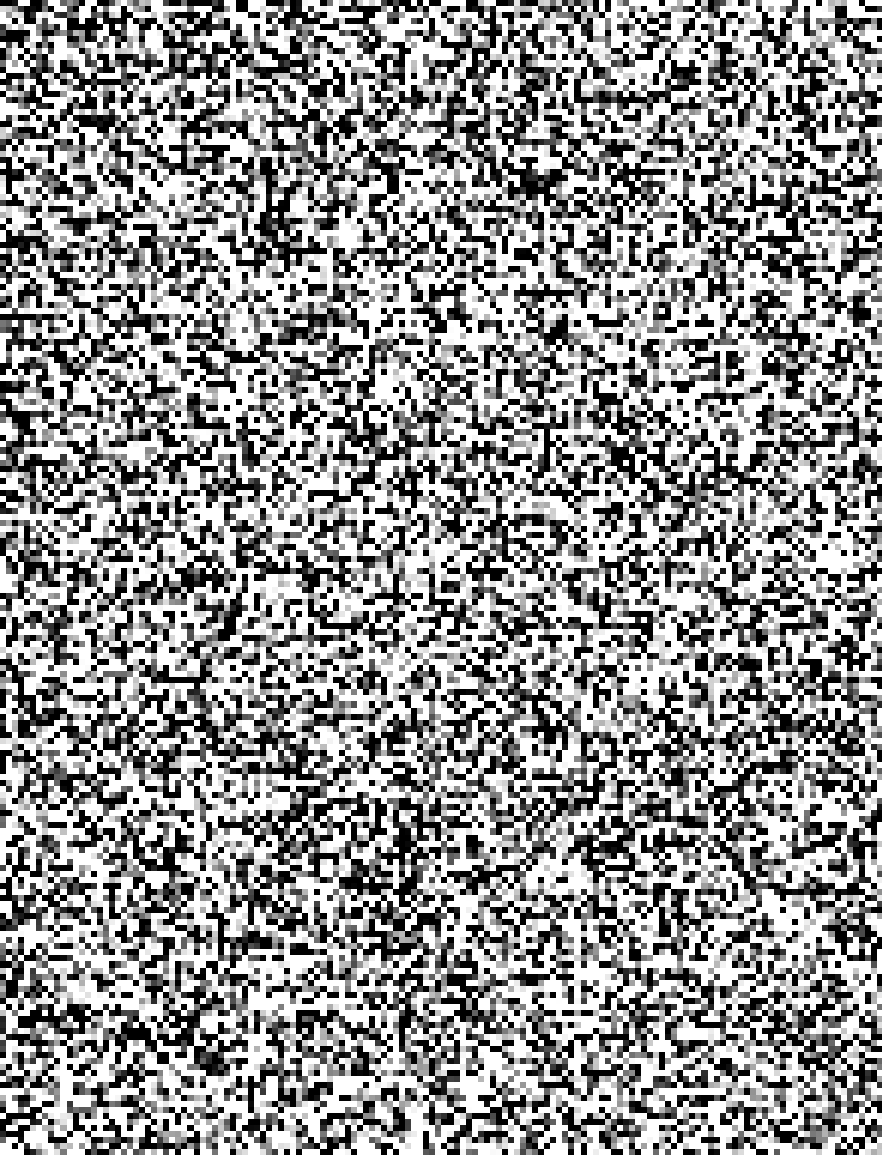}
        &
        \includegraphics[width=0.24\textwidth]{
            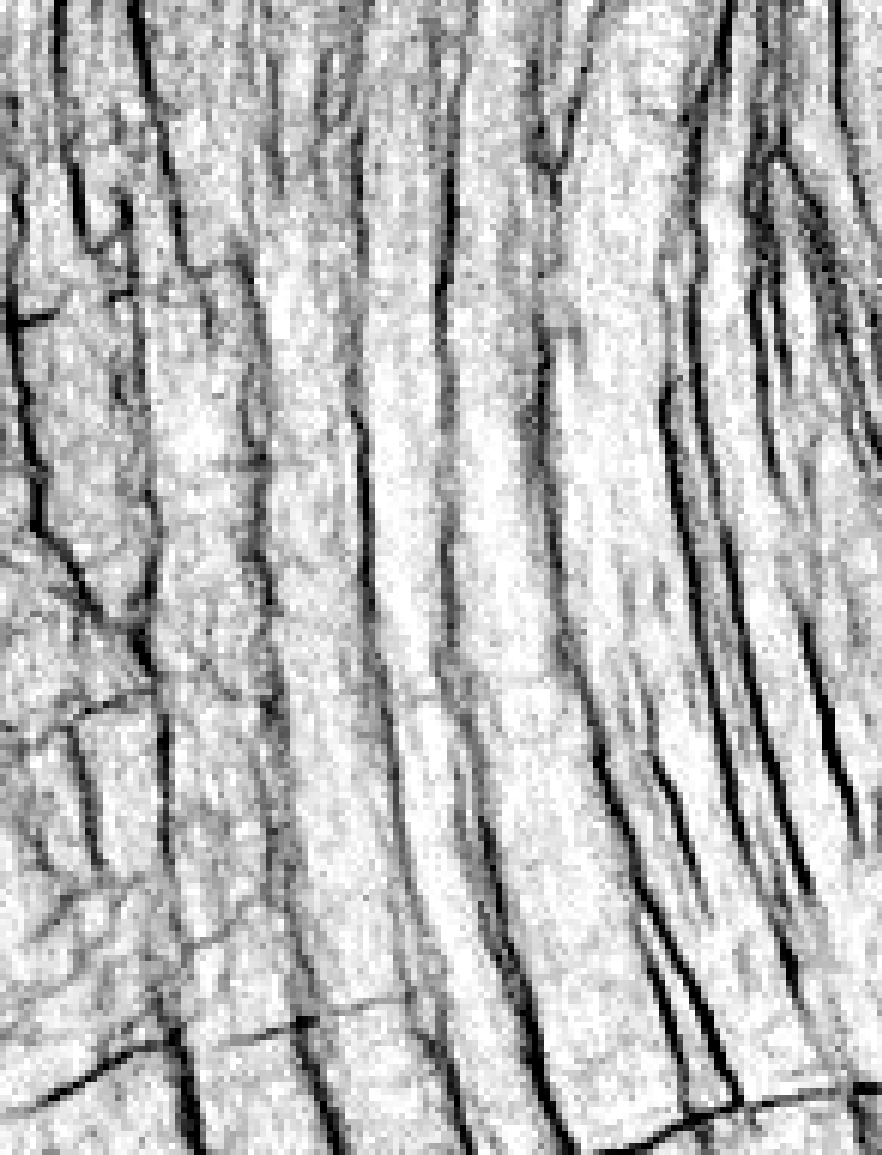}
    \end{tabular}

    \caption{
        Representative seismic reconstructions with $\rho=0.25$ and
        $\alpha=0.10$. From top to bottom, the rows show time-sample
        slices $16$, $27$, and $38$. All panels in the same row use a
        common intensity range determined from the corresponding
        ground-truth slice.
    }
    \label{fig:seismic-reconstruction}
\end{figure}

\FloatBarrier

\section{Conclusion}

This work studied robust low-tubal-rank tensor completion under tensor cross-concentrated sampling (t-CCS) in the presence of sparse gross outliers. To address this problem, we proposed Robust Iterative t-CUR (R-ItCUR), a tensor-native algorithm that combines adaptive blockwise Welsch correction with projected blockwise gradient descent.By maintaining an implicit t-CUR representation throughout the iterations, R-ItCUR avoids reconstructing the ambient tensor and eliminates the need for full-tensor t-SVD computations.

Extensive experiments on synthetic tensors, cardiac MRI data, and three-dimensional seismic data demonstrate the effectiveness and robustness of the proposed method. On synthetic problems, R-ItCUR exhibited stable convergence behavior and accurately separated sparse corruptions from the underlying low-tubal-rank structure across a range of tubal ranks and corruption levels. On the real datasets, R-ItCUR achieved reconstruction quality comparable to ITCURTC in the absence of outliers while exhibiting substantially greater robustness as the corruption rate increased. Under identical t-CCS observations, it consistently outperformed competing methods designed for robust data recovery.
The experiments also highlight the importance of explicitly exploiting the observation geometry. While generic robust completion methods can perform competitively under uniform sampling, their performance deteriorates under cross-concentrated sampling patterns. In contrast, R-ItCUR is specifically designed for the t-CCS setting and directly leverages the sampled tensor cross structure, resulting in improved robustness and reconstruction accuracy.

Overall, the results demonstrate that robust low-tubal-rank tensor completion can be performed effectively under t-CCS observations. The proposed framework provides a practical and scalable solution for robust tensor recovery from cross-concentrated samples, even in the presence of sparse gross corruptions.

\section*{Acknowledgments}
 This work is based upon work supported by the National Science Foundation under Grant No. DMS-1929284 while the authors were in residence at the Institute for Computational and Experimental Research in Mathematics in Providence, RI, during the ``Tensor Analysis for Large-Scale Data'' Colloborate@ICERM program. The work of the authors is partially supported by the National Science Foundation under Grant No. DMS-2603463, DMS-2603464, DMS-2607997, the Cancer Prevention and Research Institute of Texas under Grant RP260222, the National Institutes of Health under Grant K25CA317042, and the Innovation in Cancer Informatics Fund. On behalf of all authors, the corresponding author states that there is no conflict of interest.

\bibliographystyle{unsrt}
\bibliography{references}
\end{document}